\documentclass[11pt]{article}

\usepackage[utf8]{inputenc}
\usepackage[T1]{fontenc}
\usepackage{lmodern}

\usepackage{amsmath, amssymb, amsthm}
\usepackage{graphicx}
\usepackage{hyperref}
\usepackage{url}
\usepackage{natbib}
\usepackage{geometry}
\usepackage{algorithm}
\usepackage{algorithmic}

\usepackage{hyperref}
\usepackage{mathtools}
\usepackage{amsmath}
\usepackage{amsthm}
\usepackage{amssymb}
\usepackage{setspace}
\usepackage{booktabs}
\usepackage{tabularx}
\usepackage{bm}
\usepackage{xcolor}
\definecolor{magenta}{rgb}{1.0, 0.0, 1.0}
\usepackage{subcaption}

\newcommand{\junpei}[1]{}

\newcommand{\updatedafterneurips}[1]{#1}
\newcommand{\updatedduringaistats}[1]{#1}

\makeatletter

\usepackage{amsthm}
\usepackage{amsfonts}
\usepackage{bbm}
\usepackage{enumerate}
\usepackage{float}
\usepackage{caption}

\theoremstyle{remark}

\theoremstyle{definition}
\newtheorem{thm}{Theorem}
\newtheorem{lem}[thm]{Lemma}
\newtheorem{definition}{Definition}

\newtheorem{remark}{Remark}

\newtheorem{claim}{Claim}

\usepackage{amsmath} 
\usepackage{amssymb}
\usepackage{bm}
\usepackage{ascmac}
\usepackage{setspace}
\usepackage[at]{easylist}
\usepackage{comment}

\usepackage{url}

\usepackage{breakurl}

\usepackage{amsmath}
\DeclareMathOperator*{\argmax}{arg\,max}

\newcommand{\Real}{\mathbb{R}}
\newcommand{\Natural}{\mathbb{N}}

\newcommand{\Ex}{\mathbb{E}}
\newcommand{\Prob}{\mathbb{P}}

\newcommand{\Ind}{\bm{1}}

\newcommand{\nn}{\nonumber\\}

\newcommand{\Ist}{\mathcal{I}^*}

\newcommand{\E}{\Ex}

\newcommand{\EA}{\mathcal{A}}

\newcommand{\EN}{\mathcal{N}}

\newcommand{\EP}{\mathcal{P}}
\newcommand{\EQ}{\mathcal{Q}}

\newcommand{\bx}{\bm{x}}

\newcommand{\bw}{\bm{w}}

\newcommand{\bP}{\bm{P}}

\newcommand{\bQ}{\bm{Q}}

\newcommand{\hatbQ}{\bQ}

\newcommand{\Cfull}{R^{\mathrm{go}}}

\newcommand{\e}{\mathrm{e}}

\newcommand{\PoE}{\mathrm{PoE}}

\newcommand{\ZHone}{Z}

\newcommand{\simplex}[1]{\Delta^{(#1)}}
\newcommand{\Realp}{\Real^{+}}

\newcommand{\KL}[2]{D\left({#1} \Vert {#2}\right)}  % KL(a||b)

\newcommand{\paren}[1]{\mathopen{}\left( {#1}_{{}_{}}\,\negthickspace\right)\mathclose{}}

\allowdisplaybreaks

\makeatother

\newcommand{\Csuf}{{C_{\mathrm{suf}}}}
\newcommand{\Cstability}{{C_{\mathrm{stbl}}}}
\newcommand{\Cstabilityinner}{\Cstability}
\newcommand{\Ccounting}{{C_{\mathrm{cn}}}}

\newcommand{\gkldiv}[2]{\KL{#1}{#2}}

\newcommand{\Ctranstwo}{C_{\mathrm{trans2}}}
\newcommand{\Cweightlower}{w_{\min}}

\newcommand{\calS}{\mathcal{S}}
\newcommand{\iset}[1]{\calS^{(#1)}}
\newcommand{\perr}{\PoE}
\newcommand{\olog}{\overline{\log}}

\newcommand{\modP}[1]{\bar{P}^{\mathrm{mod}}(#1)} % Todo use for representing Junya's modified mean
\newcommand{\fracalloc}{{\mathrm{int}}}
\newcommand{\fracwrap}[1]{\lceil{#1}\rceil^{\fracalloc}}
\newcommand{\Nbatch}{N_B}
\newcommand{\Runnum}{10,000}
\newcommand{\Stbl}{S} %replacing V(Q, P)...

\usepackage{enumitem}
\setlist[itemize]{leftmargin=1.5em}   % デフォルトは約2.5 em
\setlist[itemize,2]{leftmargin=2em}   % leftmargin を好きな値に

\usepackage{thmtools}
\usepackage{thm-restate}

\usepackage{wrapfig}   % wrapfigure 環境

\usepackage{tabularx, booktabs, array}  % notation table

\usepackage[normalem]{ulem}

\newif\ifrestatement

\begin{document}

\title{Rate-optimal Design for Anytime Best Arm Identification}

\author{
Junpei Komiyama$^{1,4,5}$\footnote{Based on research conducted at NYU Stern. Email: junpei@komiyama.info}, Kyoungseok Jang$^{2}$, Junya Honda$^{3,5}$ \\
$^1$Stern School of Business, New York University \\
$^2$Department of AI, Chung-Ang University \\
$^3$Graduate School of Informatics, Kyoto University\\
$^4$Machine Learning Department, Mohamed bin Zayed University of Artificial Intelligence \\
$^5$RIKEN AIP \\
}

\maketitle

\begin{abstract}
We consider the best arm identification problem, where the goal is to identify the arm with the highest mean reward from a set of $K$ arms under a limited sampling budget. This problem models many practical scenarios such as A/B testing. 
We consider a class of algorithms for this problem, which is provably minimax optimal up to a constant factor. This idea is a generalization of existing works in fixed-budget best arm identification, which are limited to a particular choice of risk measures.  
Based on the framework, we propose Almost Tracking, a closed-form algorithm that has a provable guarantee on the popular risk measure. 
Unlike existing algorithms, Almost Tracking does not require the total budget in advance nor does it need to discard a significant part of samples, which gives a practical advantage. 
Through experiments on synthetic and real-world datasets, we show that our algorithm outperforms existing anytime algorithms as well as fixed-budget algorithms.
Our recommended algorithm for practitioners is found in the final section.
\end{abstract}

\section{Introduction}

In the Best Arm Identification (BAI) problem, a learner sequentially samples from $K$ arms with unknown means $P_1,\dots,P_K$ and wishes to output the arm with the largest mean. This classical pure exploration task underpins applications in A/B testing \citep{LiCLS10,peeking,DBLP:conf/aistats/XiongCFJJ24}, clinical trials \citep{villar,JMLR:v22:19-228,Wang2023Adaptive}, and adaptive control \citep{DBLP:conf/iros/KovalKPS15,DBLP:conf/iros/PaudelS23}.
The fixed-budget setting, which aims to maximize identification accuracy after exactly $T$ samples, captures scenarios where evaluation may be stopped externally (traffic shifts, budget exhaustion). 
Formally, an anytime algorithm consists of
\begin{itemize}
  \setlength{\parskip}{-0.15cm}
  \setlength{\itemsep}{0.25cm}
\item Sampling strategy, which selects the \textit{sampling arm} $I(t)$ given the history up to round $t-1$.
\item Recommendation strategy, which selects the \textit{recommendation arm} $J(t)$, which is the estimate of the best arm given the history up to round $t$.
\end{itemize}

Popular fixed-budget algorithms (SR \citep{Audibert10}, SH \citep{Karnin2013}) rely on elimination schedules that presuppose $T$; anytime adaptations via doubling \citep{DBLP:conf/icml/ZhaoSSJ23} sacrifice statistical efficiency by discarding earlier samples. We propose an alternative, Almost Tracking, that (i) never requires the horizon $T$, (ii) avoids elimination, and (iii) retains all observations while (iv) has provable minimax rates as described below.

A natural performance measure of a BAI algorithm is the probability of error $\PoE(\bP) := \mathbb{P}(J(T) \neq i^*(\bP))$, which is the probability that the recommendation arm $J(T)$ does not match the true best arm\footnote{In this paper, 
we identify the mean of the distribution $P_i$ with the distribution itself, and
assume the uniqueness of $i^*(\bP)$.} $i^*(\bP) = \argmax_i P_i$ when the (unknown) true model is $\bP$.
The smaller $\PoE(\bP)$ is, the better the algorithm is.
On this performance measure, it is generally important to consider the minimax optimality.
This is because the performance of an algorithm typically involves a trade-off across different problem instances.\footnote{This contrasts with uniform optimality in the fixed-confidence setting; see Section~\ref{sec:related_work} for a comparison.}
Specifically, an algorithm that performs exceptionally well for one problem instance $\bP$ may perform poorly for another problem instance $\bP'$ \citep{WangAP24}. 
Consequently, in fixed-budget BAI, there is no universal notion of optimality.
In view of this, minimax optimality is the following objective: Given a risk measure, minimize the risk of the algorithm in the worst-case instance. In fact, existing line of works \citep{Audibert10,Karnin2013,Carpentier2016,DBLP:conf/icml/ZhaoSSJ23} can be viewed as special classes of minimax optimal algorithms with a particular choice of risk measure (Section \ref{sec:related_work}).

In this context, \cite{komiyama2022minimax} derived a fundamental lower bound for any given risk measure, which implies that the optimal exploration allocation--i.e., the sampling proportions across arms--should be a function of the empirical means.
However, establishing the corresponding theoretical rates for such an adaptive allocation is generally challenging due to the trackability issues discussed below, and hence most existing algorithms resort to less flexible exploration schedules than the theoretical optimum.
A notable exception is the minimax-optimal algorithm called Delayed Optimal Tracking (DOT, \citealt{komiyama2022minimax}). However, DOT is computationally intractable. It needs to solve a dynamic programming, and each step of the dynamic programming involves a non-convex optimization over $K$ functions.\footnote{See Section \ref{sec_nonconvexity} for the non-convexity of the objective.} We desire a more practical algorithm that is easy to compute and implement.

\paragraph{Summary of our results:} 
\begin{itemize}
    \item We establish a constant-factor approximation to a one-shot game optimum that links adaptive allocations to the minimax rate (Section \ref{sec_characterizing}). Based on the one-shot game, we introduce Simple Tracking.
    \item However, analysis of Simple Tracking is very challenging. We further propose Almost Tracking, a closed-form, anytime sampling rule that provably attains minimax rate optimality up to a constant-factor without knowing $T$, without discarding samples, and without any horizon-calibrated hyperparameter (Section \ref{sec_trackability}).
    \item
    Almost Tracking requires a constant-factor approximation of the optimal allocation function. For the canonical $H_1$ risk measure \citep{Audibert10}, we derived a closed-form allocation formula (Section \ref{sec_optimization}).
    \item In our synthetic benchmarks and real datasets (OBD, MovieLens), Almost Tracking outperforms anytime (DSH/DSR) and fixed-budget baselines (SR/SH/CR), despite using neither the budget $T$ nor $T$-dependent tuning (Section \ref{sec:experiments}).
    \item As a byproduct, we derive a tight characterization of Successive Rejects by establishing its exact error exponent (Section \ref{sec:sr_comparison}). %
\end{itemize}
\updatedduringaistats{
After these sections, we conclude the paper (Section \ref{sec:conclusion}).
Our recommended algorithm for practitioners is found in the following section (Section \ref{sec_practitioners}).
}

\subsection{Minimax optimality}

The raw worst-case $\PoE$ is not very meaningful since the identification of the best arm can be arbitrarily hard when the gap between the best arm and the other arms is small.
A reasonable target to optimize is the minimax rate \citep{komiyama2022minimax,Degenne23} of the risk.
Let $H(\bP): \EP^K \rightarrow \Realp \cup \{\infty\}$ be a risk measure, where $\EP$ is the set of distributions.
We then discuss the best possible \textit{rate} $R>0$ such that
\[
\PoE(\bP) \le \exp\left(-\frac{RT}{H(\bP)}+o(T)\right) \text{for all $\bP$}.
\]

\begin{definition}[Rate]\label{def_rate}
We use $\PoE_{\EA}(\bP)$ to represent the probability of error when we use algorithm $\EA$. 
The \updatedduringaistats{(normalized)} rate of algorithm $\EA$ is defined as
\[
R_{\EA} := \inf_{\bP \in \Real^K: |i^*(\bP)| = 1} H(\bP) \liminf_{T\to\infty} \frac{1}{T} \log (1/\PoE_{\EA}(\bP)).
\]
\end{definition}
Thus $R_{\mathcal{A}}$ is the largest uniform constant in the exponent after normalizing by instance complexity $H(P)$. 
Since it takes the minimum over all instances of $K$ arms, the rate $R_{\EA}$ depends on $K$.
For the most widely-studied risk measure $H_1(\bP)$ (Section \ref{sec_optimization}), the best algorithm has $R_{\EA} = \Theta(1/(\log K))$.
\begin{definition}{\rm (Rate-optimality)}
An algorithm is \textit{rate-optimal} (or minimax optimal up to a constant factor) if there exists a universal constant\footnote{A constant is universal if it is independent of model parameters, such as $K$ and $\EP$.} $C>0$ such that, for any $K$ and for any other algorithm $\EA^*$, $R_{\EA} \ge C R_{\EA^*}$ holds.
\end{definition}

\subsection{Large deviation}

Let $n_i$ be the number of samples on arm $i$. It is known from large deviation theory \citep{dz2010} that the empirical mean $Q_i$ is approximately equal to $Q$ with probability roughly given by
\[
\mathbb{P}[Q_i \approx Q]  = \mathrm{poly}(n_i)\exp\left( - n_i D(Q \Vert  P_i) \right),
\]
where $\mathrm{poly}(n_i)$ is a polynomial factor and $D(\cdot \Vert \cdot)$ is the Kullback-Leibler (KL) divergence.

It applies to a broad class of distributions, including members of the exponential family such as Bernoulli, Multinomial, and Gaussian distributions. Accordingly, our results in Section~\ref{sec_characterizing} extend to any distribution class that admits such inequalities. From Section~\ref{sec_trackability}, we focus on Gaussian rewards\footnote{Throughout our analysis, we assume unit variance ($\sigma = 1$). The algorithm itself is variance-agnostic; only the error rate scales with $\sigma^2$.}, where the KL divergence simplifies to $\gkldiv{Q_i}{P_i} = (Q_i - P_i)^2 / 2$. Since the Gaussian large deviation bound can also be used as an upper bound for any sub-Gaussian distribution\footnote{A random variable $X$ is sub-Gaussian if 
$
\mathbb{E}\bigl[e^{(X-\mathbb{E}[X])t}\bigr]\le
\exp\!\bigl(\tfrac{t^{2}}{2}\bigr)
$ for all $t>0$.}, our algorithms apply to the class of sub-Gaussian rewards.
Appendix \ref{sec:limitations} describes limitations of each theorem.

\subsection{Related work}\label{sec:related_work}

The study of identifying the best arm began with seminal work \citep{bechhofer1954single} on ranking-and-selection in the 1950s. This line of research evolved into what became known as ordinal optimization \citep{chen2000,glynn2004large} and later influenced the works called the best arm identification (BAI, \cite{Audibert10}).

\paragraph{Fixed-confidence BAI} In the fixed-confidence setting \citep{jamiesonsurvey}, the learning agent decides whether to stop additional sampling or continue sampling for more information. The objective is to minimize the expected stopping time with a constraint that $\PoE$ at the end of the experiment should be smaller than a given confidence level $\delta\in(0,1)$. Two classes of algorithms, called Track and Stop algorithms \citep{kaufman16a,DegenneK19} and Top-Two algorithms \citep{Russo2016,jourdan2022top,DBLP:conf/acml/LeeHS23,DBLP:journals/jsait/MukherjeeT23,DBLP:conf/colt/YouQWY23,DBLP:conf/nips/Bandyopadhyay0A24}, are well-known. Both algorithms are uniformly optimal across all instances.
However, any such uniformly optimal algorithm in FC-BAI has arbitrarily worse performance in the FB-BAI setting \cite[Appendix C]{komiyama2022minimax}.

\paragraph{Fixed-budget BAI} 

The first paper on the modern formulation of BAI \citep{Audibert10} considers the fixed-budget setting where one is interested in the performance of an algorithm with a fixed number of samples. They also propose the successive reject (SR) algorithm. From the perspective of minimax optimality, their analysis is based on two particular complexity measures $H_1$ and $H_2$. \cite{Karnin2013} introduced an algorithm called successive halving (SH) that has the same rate as SR up to a constant factor. \cite{Carpentier2016} show that SR and SH are minimax optimal up to a constant factor for $H_1$. \cite{Shahin2017} show extensive empirical comparison of these algorithms. Recently, \cite{wang2023best} proposed a new algorithm called Continuous Rejection (CR) and reported that it outperforms SR and SH. Note that, all of these algorithms are elimination-based and requires knowing the budget $T$ in advance. \cite{DBLP:conf/icml/ZhaoSSJ23} introduce a doubling trick to SH to make it anytime. SR and SH are popular, and the idea of gradually eliminating the candidates arms is extended into many pure exploration problems.
To name a few, linear bandits \citep{Soare2014,pmlr-v80-tao18a,DBLP:conf/nips/YangT22}, bilinear models \citep{JangJYK21}, combinatorial bandits \citep{Chen2014,DBLP:conf/aaai/DuK021}, unimodal bandits \citep{yu2011unimodal,ghosh2024fixed}, nonparametric models \citep{DBLP:conf/alt/BarrierGS23}, contextual problems \citep{DBLP:conf/nips/LiRNJJ22}, and robust and nonstationary problems \citep{DBLP:conf/uai/YuSLK18,DBLP:journals/corr/abs-2109-04941,DBLP:conf/aistats/TakemoriUG24}. 
In this paper, we propose an algorithm that fundamentally diverges from existing elimination-based approaches and has the potential to influence the design of best arm identification algorithms across various extended settings, with the value of anytime and provable minimax optimality.

\section{Characterizing optimization}
\label{sec_characterizing}
We can bound the minimax optimal rate from above (that is, bound the PoE from below) by considering the following game between an agent and an adversary.
\begin{definition}{\rm (One-shot game)}\label{def_oneshotgame}
Information structure and move order:
\begin{enumerate}
    \setlength{\parskip}{-0.1cm}
    \setlength{\itemsep}{0.15cm}
    \item The agent \emph{commits ex ante} to a measurable mapping $\bw:\EQ^K \to \simplex{K}$ (no knowledge of the particular $(\bP,\bQ)$ realization).
    \item Nature selects a pair $(\bP,\bQ)$ with $\bP\in \EP^K: |i^*(\bP)|=1$, $\bQ\in \EQ^K$.
    \item The mapping is evaluated at $\bQ$ producing allocation $\bw(\bQ)$; the agent suffers loss $L(\bP,\bQ,\bw)= H(\bP)\sum_{i} w_i(\bQ) D(Q_i\Vert P_i)$ when $i^*(\bP) \notin i^*(\bQ)$.
\end{enumerate}
Here,
$\simplex{K}\subset [0,1]^K$ is the $(K-1)$-dimensional probability simplex,
and $\EQ, \EP$ are the space of empirical means and true means\footnote{Though $\EQ$ and $\EP$ are identical to $\mathbb{R}$ in this paper we use these notations to clarify that they are the sets of empirical means and true means.}. 
The value of the game is $\inf_{\bP,\bQ} L$ with the supremum understood as the agent choosing the mapping before seeing $(\bP,\bQ)$. Here, $\bP$ represents the true means, $\bQ$ a hypothesized empirical means that would mislead identification (misidentification event $i^*(\bP) \notin i^*(\bQ)$), and $D(Q_i\Vert P_i)$ encodes the large deviation cost of that deception. The mapping $\bw(\bQ)$ corresponds to the allocation of samples to each arm. An agent models a BAI algorithm:
For example, $w_i(\bQ)$ corresponds to the algorithm that selects arm $i$ for $ Tw_i(\bQ(T))$ times when the empirical mean at the end of the round $T$ is $\bQ(T)$.
\end{definition}

The rate of any algorithm can be bounded by the worst-case loss of this game: %
\begin{thm}[Theorem 1 in \cite{komiyama2022minimax}]\label{thm_lower}
Under any algorithm $\EA$ it holds that 
\begin{equation}
R_{\EA} \le \sup_{\bw(\cdot):\EQ^K\to \simplex{K}}\,
\inf_{\bQ, \bP: \Ist(\bP) \notin \Ist(\bQ)} H(\bP)
\sum_{i\in[K]} w_i(\bQ) D(Q_i\Vert P_i)
=:\Cfull.
\label{def_rgo} %
\end{equation}
\end{thm}

While the one-shot ``game-optimal'' value $\Cfull$ upper bounds any attainable rate, it is a priori unclear whether it is itself (even approximately) achievable by a practical algorithm.
This is because the above game considers a setting which is favorable for the agent
where the agent can determine the allocation $w(\bQ)$ of plays after observing the entire empirical distributions $\bQ$,
which is not possible in the actual BAI task.

One naive idea to achieve $\Cfull$ is to greedily \text{track} the optimal allocation,
that is, to track the solution of the characterizing optimization of Eq.~\eqref{def_rgo}, which we call simple tracking in Algorithm~\ref{alg:R_track}.
\begin{algorithm}[t]
\caption{Simple Tracking}
\label{alg:R_track}
\begin{algorithmic}[1] 
\STATE Input: Target weight function $\bw(\bQ)$.
\STATE Draw each arm once.
\FOR{$t = K+1, 2, \ldots, T$}
    \STATE Draw arm $I(t) = \argmax_{i\in[K]} \left\{ w_{i}(\hatbQ(t-1)) - N_i(t-1)/(t-1) \right\}$, where $N_i(t-1)$ be the number of samples for arm $i$ up to round $t-1$.
\ENDFOR
\RETURN $J(T) = $ empirical best arm at the end.
\end{algorithmic}
\end{algorithm}
However, deriving a performance guarantee for Algorithm \ref{alg:R_track} with its target function set to the solution $\bw^*(\bQ)$ of the optimization problem in Theorem \ref{thm_lower} is highly nontrivial.
The major challenges are as follows.
\begin{description}
\item[Trackability issue:] It is possible that the empirical distribution of the current round $\hatbQ(t)$ drastically changes from
that of the last round $\hatbQ(t-1)$, and in this case, the allocation can be far from the optimal allocation
$\bw^*(\hatbQ(t))$.
Such an event occurs with exponentially small probability and becomes negligible in the fixed-confidence scenario where we just need to
evaluate the expected number of samples.
However, it is not the case in the fixed-budget setting because the error probability is also exponentially small
and it is possible that the error due to tracking is not negligible.

\item[Optimization issue:] At each round, Algorithm \ref{alg:R_track} computes $\bw^*(\hatbQ(t-1))$. This optimization is non-convex in general and makes this algorithm computationally intractable.
While one can precompute and store the mapping $\bw^*(\cdot)$ before the game,
doing so compromises the practical utility and the interpretability of the algorithm.
\end{description}

From the theoretical perspective, due to the trackability issue it is unclear how close $\Cfull$ is to the optimal rate.
Our first contribution is that $\Cfull$ is close to the optimal one up to a factor of two.
\begin{thm}\label{thm_upper}
There exists an algorithm $\EA$ such that
\begin{align}
\frac{\Cfull-\varepsilon}{2}\le
R_\EA
\end{align}
for any $\varepsilon > 0$.
Consequently, we have $\Cfull/2\le R_\EA \le \Cfull$.
\end{thm}
From this result, we can see that algorithms achieving rate close to $\Cfull$ are also close to the minimax optimal.

In the proof of Theorem \ref{thm_upper} in Appendix \ref{sec:two_approx}, we also give an explicit algorithm to achieve this bound.
However, from the practical viewpoint,
the algorithm used in the proof discards a significant amount of samples and its performance is not particularly promising unlike the Simple Tracking.

From these observations, we address the above two difficulties (tracking and optimization issues)
of simple tracking to provably realize a practical and (nearly) minimax-optimal algorithm.

\section{Addressing trackability issue}\label{sec_trackability}

In this section we propose a generic framework of algorithms to resolve the trackability issue.
Henceforth, we focus on sub-Gaussian reward distributions, and we use squared distance instead of the KL divergence. 
Theorem~\ref{thm_lower} discusses the best possible error rate (in the sense of minimax optimality) for a stronger model of the agent that can determine the allocation $\bw=\bw(\bQ)$ depending on the empirical distribution $\bQ$ as if he/she knows $\bQ$ before allocating the samples.
The actual agent does not have such knowledge and needs to track the ideal allocation $\bw(\bQ)$ based on the current empirical distribution, whose tracking error can essentially affect the error rate.

To alleviate this difficulty, we introduce the following algorithmic and theoretical tricks.
\begin{itemize}
  \setlength{\parskip}{-0.15cm}
  \setlength{\itemsep}{0.25cm}
\item A batched algorithm assigning samples to each arm without being affected by tracking errors of other arms.
\item Error probability analysis incorporating the variance of empirical distributions.
\end{itemize}
One issue of the simple tracking is that if there is an arm such that the current allocation is significantly smaller than the target allocation $w_i(\bQ)$
then all the allocation is assigned to this arm until the gap is mostly filled,
which prohibits stably exploring other arms.
To avoid this issue, we propose an algorithm Almost Tracking, which is
formalized as Algorithm \ref{alg:R_track_almost}.

\begin{algorithm}[t]
    \caption{Almost Tracking}
    \label{alg:R_track_almost}
    \begin{algorithmic}[1] 
        \STATE Input: Target allocation function $\bw(\bQ)$, Parameter $\Csuf \in (0, 1)$. Draw per batch $N$.
        \STATE Draw each arm for $\fracwrap{N/K}$ times. Here, $\fracwrap{\cdot}$ is a rounding operator defined in Appendix \ref{sec:fractional_pulls}.  
        \FOR{$b=2,3,\dots$}
            \STATE Calculate\vspace{-1em}
\begin{align}\label{ineq_insufset}
 \mathcal{K}_{\mathrm{insuf}} \hspace{-0.2em} &= \hspace{-0.2em}  
\left\{i \hspace{-0.2em} : \frac{1}{b-1} \sum_{b'=1}^{b-1} w_{b',i} \hspace{-0.2em} \le \hspace{-0.2em} \frac{1}{\Csuf} w_i(\bar{\bQ}_{b-1})\right\}\\
 s_{\mathrm{insuf}} &= \sum_{i \in \mathcal{K}_{\mathrm{insuf}}}  w_i(\bar{\bQ}_{b-1}).
\end{align}
\vspace{-0.5em}
            \STATE Draw each arm for $\fracwrap{w_{b,i} N}$ times, where $w_{b,i} = 0$ for $i \notin \mathcal{K}_{\mathrm{insuf}}$, and $w_{b,i} = w^*_i(\bar{\bQ}_{b-1}) / s_{\mathrm{insuf}}$ for $i \in \mathcal{K}_{\mathrm{insuf}}$.
        \ENDFOR
        \RETURN $J^*(T) = $ empirical best arm at the end.
    \end{algorithmic}
\end{algorithm}
While Simple Tracking (Algorithm \ref{alg:R_track}) myopically aims to follow the optimal allocation $w^*(\hatbQ(t-1))$ by drawing the most insufficient arm, Almost Tracking (Algorithm \ref{alg:R_track_almost}) aims to follow the optimal allocation $w^*(\hatbQ(t-1))$ but splits efforts into all insufficiently drawn arms in a batched manner.
To be more specific, we compute the list $\mathcal{K}_{\mathrm{insuf}}$ of insufficiently sampled arms
and allocate fractions of
\begin{equation}\label{eq_sampling_propto}
w_{b,i} =
\begin{cases}
w_i(\bar{\bQ}_{b-1}) / Z_b  & \text{if $i \in \mathcal{K}_{\mathrm{insuf}}$}, \\
0   & \text{otherwise},
\end{cases}
\end{equation}
where $\bar{\bQ}_{b-1}$ is the empirical mean up to batch $b-1$, $Z_b = \sum_{i \in \mathcal{K}_{\mathrm{insuf}}} w_i(\bar{\bQ}_{b-1})$ is the normalization factor. The model parameter $\Csuf \in (0,1)$ is supposed to be close to $1$, in which case $\mathcal{K}_{\mathrm{insuf}}$ just lists up all arms with its number of samples below the target allocation $w^*_i(\bar{Q}_{b-1})$.

Despite the structure of Almost Tracking, we cannot completely ensure the closeness between the final allocation and the ideal allocation $\bw(\bQ)$
when $\bar{\bQ}_b$ has drastically changed with batches.
To evaluate the error due to this gap of the allocation,
we incorporate the variance of the empirical means in the analysis.
According to the large deviation principle, roughly speaking, if we allocate $N_{1,i}, N_{2,i},\dots,N_{B,i}$ samples to arm $i$
in each batch, then the sequence of empirical means $(Q_{1,i}, Q_{2,i},\dots,Q_{B,i})$ appears with probability
\begin{align}
\lefteqn{
\Pr[(\mbox{empirical distributions are $(Q_{1,i}, Q_{2,i},\dots,Q_{B,i})$})]
}\nn
&\approx
\exp\left(
- \frac{1}{2}\sum_{b\in[B]} N_{b,i} (Q_{b,i} - P_i)^2
\right)
\text{\ \  (large deviation)}
\label{ineq_largedevtwo}
\\
&\approx
\underbrace{
\exp
\left(
-
\frac{1}{2}\left(\sum_{b\in[B]} N_{b,i}\right)
(\bar{Q}_{B,i}- P_i)^2
\right)
}_{\mbox{\small likelihood for overall empirical mean}} \times
\underbrace{\exp\left(
-\frac{1}{2}\left(
\sum_{b\in[B]} N_{b,i}(Q_{b,i}-\bar{Q}_{B,i})^2
\right)
\right)}_{\mbox{\small likelihood for variance of empirical means}}
\label{ineq_twoterms}
\end{align}
by using the overall empirical distribution $\bar{Q}_{B,i}$ over all batches.
This decomposition means that, empirical means $Q_{b,i}$ drastically changing with $b$ is less likely to occur
compared with the empirical means $\bar{Q}_{b,i}$ constant over batches.
Our theoretical contribution is that error due to the failure of tracking for drastically changing $Q_{b,i}$ can be compensated by this variance term.

The following describes the performance guarantee for the proposed algorithm. 
\begin{restatable}{thm}{batchtrack}{\rm (Trackability)}\label{thm_batchtrack}
Let $w_i(\bQ) \ge \Cweightlower$ holds for any $\bQ, i$ for some value $\Cweightlower = \Cweightlower(K)>0$ that only depends on $K$. Let $B := \lfloor T/N \rfloor \ge 2/\Cweightlower$.
Consider Algorithm \ref{alg:R_track_almost} with $\Csuf \in (0,1)$. 
There exists a universal constant $C_{\mathrm{track}} > 0$ such that the following holds for any sequence of $\bQ_1,\bQ_2,\dots,\bQ_B$:
\ifrestatement
\begin{equation}\label{ineq_track_a}
\frac{1}{B}\sum_{b \in [B]} \sum_{i \in [K]} w_{b,i} \frac{(Q_{b,i} - P_i)^2}{2} 
\ge C_{\mathrm{track}} \inf_{\bQ: i^*(\bP) \notin i^*(\bQ)} \left( \sum_{i \in [K]} w_i(\bQ) (Q_i - P_i)^2/2 \right).
\end{equation}
\else
\begin{equation}\label{ineq_track_b}
\frac{1}{B}\sum_{b \in [B]} \sum_{i \in [K]} w_{b,i} \frac{(Q_{b,i} - P_i)^2}{2} 
\ge C_{\mathrm{track}} \inf_{\bQ: i^*(\bP) \notin i^*(\bQ)} \left( \sum_{i \in [K]} w_i(\bQ) (Q_i - P_i)^2/2 \right).
\end{equation}
\fi
\end{restatable}
By using \updatedduringaistats{$N_{b,i} = (T/B) w_{b,i}$} and decomposition of Eq.~\eqref{ineq_twoterms}, Theorem \ref{thm_batchtrack} lower-bounds Eq.~\eqref{ineq_largedevtwo}  
by a worst-case single mean $\bQ$. This bridges the gap between the one-shot game (Definition \ref{def_oneshotgame}) and the inherently sequential nature of the best arm identification problem.

Theorem \ref{thm_batchtrack} implies that, if we have a constant-factor approximation allocation, we can achieve rate-optimality.
\begin{definition}{\rm (constant-factor approximation)}\label{def_constantfactor}
An allocation $\bw(\cdot)$ is a constant-factor approximation if
there exists a universal constant $\Cstability \in(0,1)$ independent of $K$ such that
\begin{equation}    
\inf_{\bQ,\bP: i^*(\bP) \notin i^*(\bQ)} H_1(\bP)
\sum_{i\in[K]} w_i(\bQ) \frac{(Q_i - P_i)^2}{2}
\ge \Cstability \Cfull.
\end{equation}
\end{definition}
\begin{thm}{\rm (Rate-optimality of Algorithm \ref{alg:R_track_almost})}\label{thm_rateopt_gen}
Let $w_i(\bQ) \ge \Cweightlower$ holds for any $\bQ, i$ for some $\Cweightlower>0$. Let $B := \lfloor T/N \rfloor \ge 2/\Cweightlower$ and $N = \Omega(K (\log K)^3)$.
Then, for a sufficiently large $K$, Algorithm \ref{alg:R_track_almost} with $\Csuf \in (0,1)$ with a constant-factor approximation allocation $\bw$ is rate-optimal.
\end{thm}

\begin{proof}[Proof of Theorem \ref{thm_rateopt_gen}]
Due to page limitation, the lemmas are deferred to the appendix.
For any $\bP$ with $|i^*(\bP)|=1$ and any $\bQ: i^*(\bQ) \ne i^*(\bP)$, it always holds that
\begin{align}
\lefteqn{
\sum_{b \in [B]} \sum_{i \in [K]} N_{b,i} \frac{(Q_{b,i}-P_i)^2}{2}
}\\
&\ge C_{\mathrm{track}} \inf_{\bQ} \sum_{i \in [K]} w_i(\bQ) \frac{(Q_i -P_i)^2}{2} T
\tag{by Theorem \ref{thm_batchtrack}}\\
&\ge 
\frac{C \Cfull T}{ H_1(\bm{P})}.
\tag{by definition of constant-factor approximation}%
\\
\label{ineq_prob_hone}
\end{align}
for $C = C_{\mathrm{track}} \Cstability >0$.
Therefore, we have
\begin{align}
\lefteqn{
\Prob \left[ i^*(\bQ) \ne i^*(\bP) \right]
}\\
&\le \Prob \left[\sum_{b \in [B]} \sum_{i \in [K]} N_{b,i} \frac{(Q_{b,i}-P_i)^2}{2} \ge \frac{C \Cfull T}{ H_1(\bm{P})} \right]
\tag{by Eq.~\eqref{ineq_prob_hone}}\\
&\le 
\updatedduringaistats{
\exp\left(
- C \left(1 - O\left(\frac{1}{\log K}\right)\right) \frac{\Cfull T}{ H_1(\bm{P})} 
+ o(T)
\right).\tag{by Lemma \ref{lem_abstpoe_choosen}, which uses large deviation}
}
\\\label{ineq_rate_final}
\end{align}
\end{proof}

\begin{remark}{\rm (Finite-time property)}
While our rate is asymptotic, the proof itself does not fundamentally rely on asymptotics. By carefully tracing the argument, one could in fact derive a finite-time bound. Namely, it is possible to express the $o(T)$ term in Eq.~\eqref{ineq_rate_final} explicitly as a function of the model parameters. We do not pursue this direction here, however, since the resulting expression would be unnecessarily cumbersome.
\end{remark}

\section{Addressing optimization issue}\label{sec_optimization}

Theorem \ref{thm_rateopt_gen} ensures the rate-optimality of Algorithm~\ref{alg:R_track_almost} with a constant-factor approximation.
However there remains a challenge to find a good allocation function $\bw(\bQ)$, since its optimization (Eq.~\eqref{def_rgo}) is non-convex and thus computationally costly.

\begin{figure}[t]
    \centering
\includegraphics[width=0.38\textwidth]{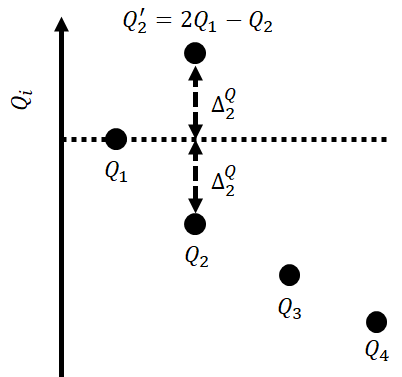}
    \caption{Hypothetical $\bQ'$ such that $H_1'(i, \bQ):=H_1(\bQ')$ with $i=2$. $Q_i'=Q_i$ for $i \ne 2$.}
    \label{fig:sample}
\end{figure}%

As discussed in the introduction, the minimax optimality and corresponding allocation depend on the complexity measure.
In this section, we focus on one of the most standard complexity measure
$
H_1(\bP) = \sum_{i \ne i^*(\bP)} (P^* - P_i)^{-2}.
$
In particular, let
\begin{equation}\label{eq_weight_hone}
w_i(\bQ) := \frac{1}{D_i(\bQ)\ZHone(\bQ)},
\end{equation}
where, for $\Delta_i^{Q} = \max_j Q_j - Q_i = Q^* - Q_i$,
\begin{align}
H_1'(i, \bQ) &= \sum_{j \ne i} \frac{1}{(\Delta_j^{Q} + \Delta_i^{Q})^2},\\
D_i(\bQ) &= 
\begin{cases} %
(\Delta_i^Q)^2 H_1'(i, \bQ) & (i \notin i^*(\bQ))\\
\min_{i \notin i^*(\bQ)} D_i(\bQ) & (i \in i^*(\bQ))\\
\end{cases},
 \label{ineq_Di}
\end{align}
and $\ZHone(\bQ) = \sum_{i \in [K]} (1/D_i(\bQ))$ is the normalization factor.
$H_1'(i, \bQ)$ corresponds to $H_1(\bQ')$ for hypothetical distribution $\bQ'$ such that
$Q_i$ is replaced with $2 Q_1 - Q_i$ as illustrated in Figure~\ref{fig:sample}.

This choice of $\bw$ is inspired by \cite{Carpentier2016}, where $\bQ'$ in Figure~\ref{fig:sample} is used
to construct a lower bound for PoE with respect to $H_1$.
We designed $\bw$ to be inverse proportional to the cost of raising arm $i$ to be optimal arm ($= (\Delta_i^Q)^2$), normalized by the sample complexity when arm $i$ becomes optimal ($= H_1(\bQ')$). 
\begin{restatable}{thm}{approxopt}%
\label{thm:approxopt}
Allocation of Eq.~\eqref{eq_weight_hone} is a constant-factor approximation (Definition \ref{def_constantfactor}).
\label{thm_approxopt}
\end{restatable}
This theorem implies the
rate-optimality with respect to the risk measure $H_1$.
\begin{thm}{\rm (Rate-optimal computationally efficient algorithm)}\label{thm_rateopt_hone}
Assume that $N = \omega(K), T=\omega(NK\log K)$.
Then, Algorithm \ref{alg:R_track_almost} with $\bw$ defined in Eq.~\eqref{eq_weight_hone} is rate-optimal for $H_1(\bP)$.
\end{thm}
Theorem \ref{thm_rateopt_hone} follows from Theorem \ref{thm_approxopt} combined with Theorem \ref{thm_rateopt_gen} with $\Cweightlower = \Theta((K\log K)^{-1})$.

\begin{table*}[t]
    \centering
    \caption{Estimated minimax rates of algorithms (larger better). Our algorithms, Simple tracking and Almost tracking are abbreviated as S.Track and A.Track, respectively.
    Left (resp.~right) four algorithms are non-anytime (resp.~anytime). Results with error bars (confidence region) are reported in Appendix \ref{subsec:confidence_intervals}. Bold font highlights possible best with two-sigma confidence level. Results of Uniform and EB-TC$_{\varepsilon_0}$ are also in \ref{subsec:confidence_intervals}.}
    \label{tab:minimax_rates}
    \begin{tabular}{c||cccc|cccc|cc}
    \toprule
    Metric & SR & SH & CR-A & CR-C & DSR & DSH & S.Track & A.Track & TS-TC & EB-TC \\
    \midrule
$H_1$ & 0.456 & 0.222 & 0.435 & 0.312 & 0.234 & 0.114 & \textbf{0.537} & \textbf{0.509} & 0.396 & 0.181 \\
$H_2$ & \textbf{0.256} & 0.106 & \textbf{0.271} & 0.194 & 0.131 & 0.056 & \textbf{0.260} & 0.255 & 0.227 & 0.104 \\
    \bottomrule
\end{tabular}
\end{table*}

\section{Experiments}\label{sec:experiments}

This section reports the results of our simulations. Our implementation is available at \url{https://github.com/jkomiyama/fb_bai_publish}.
We compared the following algorithms: 
\begin{itemize}
  \setlength{\parskip}{-0.15cm}
  \setlength{\itemsep}{0.25cm}
    \item \updatedduringaistats{Uniform algorithm that draws arms in a round robin.}
    \item Fixed-budget algorithms (require $T$): SR \citep{Audibert10}, SH \citep{Karnin2013}, and two versions of CR (CR-A and CR-C) \citep{wang2023best}. %
    \item Anytime algorithms (do not require $T$): DSH \citep{DBLP:conf/icml/ZhaoSSJ23} and Double Successive Rejects (DSR) that adopt the doubling trick. Our Simple Tracking (Algorithm \ref{alg:R_track}), and Almost Tracking (Algorithm \ref{alg:R_track_almost}). %
    \item Fixed-confidence algorithms: EB-TC and TS-TC \citep{shang2020fixed,jourdan2022top}, two empirically good versions of top-two Thompson sampling \citep{Russo2020}\updatedduringaistats{, and EB-TC$_{\varepsilon_0}$ \citep{jourdaneps2023}.} These algorithms are suboptimal in the fixed-budget setting.
\end{itemize}
Details of the algorithms are in Appendix \ref{subsec:compared_algorithms}.

\subsection{Simulation results: Minimax rates}
\label{sec:minimax_simulation}

We consider ten different instances of $\bP$ with $K=40$ arms with Gaussian rewards with a unit variance.  Our goal is to maximize the performance of algorithms in the most challenging environment. 
The instances are derived from \cite{wang2023best} and we added theoretically challenging instances for SH and Tracking.
Details of the instances are in Appendix \ref{subsec:instances_minimax_rates}. 
Table \ref{tab:minimax_rates} shows the estimated minimax rates by algorithm, which we compute the worst-case rate over ten synthetic instances, where the rate for instance $\bP$ is defined as $T^{-1} H_1(\bP) \log(1/\PoE(T))$,
which corresponds to the exponent of the probability of error normalized by the complexity $H_1(\bP)$. 
As we can see, our algorithms outperform all existing algorithms, including anytime algorithms (DSH and DSR) as well as CR, SR, and SH, with respect to the risk measure $H_1(\bP)$. 
Notably, our algorithms outperform those that require $T$, despite not having access to it.
SH does not perform as well as SR because it discards the samples at the end of each batch. DSH and DSR further discard even more samples, which leads to worse performance.
Table~\ref{tab:minimax_rates} also presents the rates under $H_2(\bP) = \max_i i \Delta_{(i)}^{-2}$, the complexity measure that SR is specifically optimized for (see Section~\ref{sec:sr_comparison}).
On this measure, the gap between SR, CR and Tracking is very narrow.
This is reasonable because SR is tailored to $H_2(\bP)$, and CR has some similarities to SR.
We further elaborate on these measures in Section \ref{sec:sr_comparison}. 

TS-TC, a fixed-confidence algorithm, performs well overall even though it is not optimized for the fixed-budget setting. However, we note that the rate of EB-TC and TS-TC with a large $T$ can be arbitrarily bad (see Section \ref{subsec:two_two_suboptimality} for such an instance).

\subsection{Simulation results: Real-world datasets}

We further create instance $\bP$ based on two real-world datasets. 
The Open Bandit Dataset \citep{SaitoOBP} is a real-world logged bandit dataset collected from a fashion e-commerce platform. 
The Movielens 1M dataset \citep{movielens1m} is a dataset that includes anonymous ratings of popular movies. 
Details of the datasets are in Appendix \ref{tab:realworld_instances}.

Figure \ref{fig:two_plots} shows the results of the algorithms on real-world datasets. Tracking and AlmostTracking outperform anytime algorithms (DSH and DSR) as well as non-anytime algorithms (SR, and SH).
The only exception is CR, which performs comparably to our algorithms in the Movielens 1M dataset. CR does not perform well in Open Bandit Dataset, which we elaborate in the section \ref{subsec:kyoungseok_cr}.

\begin{figure*}[t]
    \centering
    \begin{subfigure}[b]{0.48\textwidth}
        \centering
        \includegraphics[width=\textwidth]{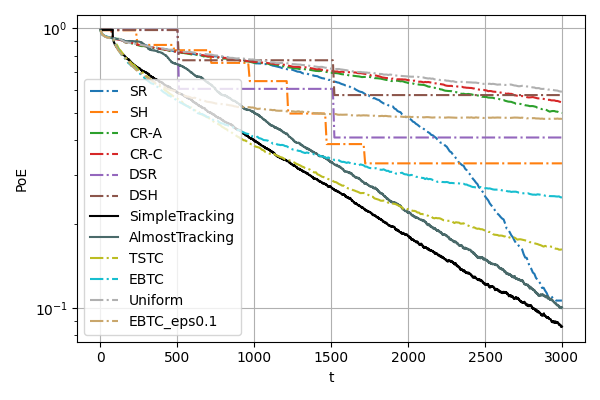}
        \caption{Open Bandit Dataset $K=80$}
        \label{fig:plot60}
    \end{subfigure}
    \hfill
    \begin{subfigure}[b]{0.48\textwidth}
        \centering
        \includegraphics[width=\textwidth]{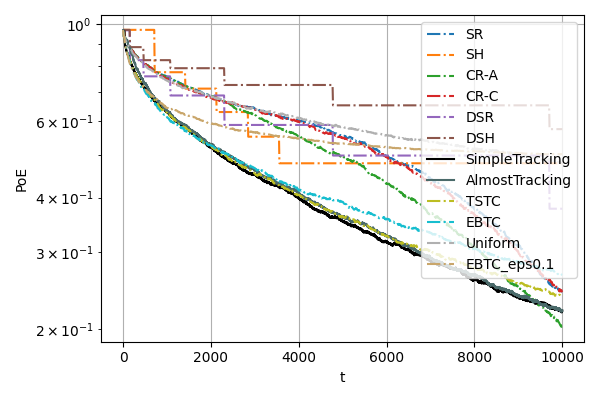}
        \caption{Movielens 1M Dataset $K=31$}
        \label{fig:plot61}
    \end{subfigure}
    \caption{Comparison of PoE across algorithms on real-world datasets (smaller better). For algorithms that do not discard samples, $J(t)$ is the empirical best arm at time $t$. For discarding algorithms (SH, DSH, DSR), $J(t)$ is the empirical best arm at the end of the most recent batch. }%
    \label{fig:two_plots}
    \vspace{-1em}
\end{figure*}

\section{Rate of successive rejects and comparison with our algorithms}
\label{sec:sr_comparison}

To compare our algorithm with SR, we show the following characterization of SR's performance.\footnote{While Theorem 2 in \cite{wang2023best} derives the same upper bound, we complement this result by showing that it also serves as a lower bound. Note that, Section 5.2 in \cite{WangAP24} derives a similar exact characterization for the case of Bernoulli rewards.}
\begin{thm}\label{thm:hthree_main}
For complexity measure $H_3(\bP)$ defined in Appendix~\ref{sec:tight_sr_analysis}, PoE of SR satisfies
\begin{equation}\label{eq_sr_error_a}
    \lim_{T \rightarrow \infty} \frac{1}{T}H_3(\bP) \olog(K)
    \log(1/\PoE(T)) 
     = 1.
     \nonumber
    \end{equation}
    Moreover, $(1/2) H_2(\bP) \le H_3(\bP) \le H_2(\bP)$ holds, where $H_2(\bP) = \max_i i \Delta_{(i)}^{-2}$ is the complexity measure proposed in \cite{Audibert10}.
\end{thm}
Theorem \ref{thm:hthree_main} tightly characterizes SR by providing matching upper and lower bounds, whereas \cite{Audibert10} offered only an upper bound.
This analysis enables us to fairly compare our algorithm with SR, eliminating the possibility of SR's analysis being loose.

\paragraph{Performance comparison between SR and Almost Tracking}

\updatedafterneurips{Theorem 1 in \cite{Carpentier2016} implies that $\Cfull = \Theta(1/(\log K))$.}
This fact, combined with our Theorem \ref{thm_rateopt_hone} implies our algorithm has PoE at most 
\begin{equation}\label{eq_our_error}
\exp\left(-\frac{CT+o(T)}{(\log K)H_1(\bP)}
\right)
\end{equation}
for some universal constant $C>0$, whereas Theorem \ref{thm:hthree_main} states that SR has PoE of 
\begin{equation}\label{eq_sr_error_b}
\exp\left(-\frac{T+o(T)}{(\olog \,K)H_3(\bP)}
\right).
\end{equation}
It holds that $H_3(\bP) \le H_1(\bP) \le H_3(\bP)(\log K)$, and thus Eq.~\eqref{eq_our_error} and Eq.~\eqref{eq_sr_error_b} are identical up to a $\log K$ factor.
Moreover, in the appendix, we show some examples where the error rate of Almost Tracking is strictly better than that of SR: 
\begin{lem}{\rm (Instance where Almost Tracking outperforms SR)}\label{lem_ourwin} %
There exists an instance such that the rate of our algorithm is $O((\log K)/(\log\log K))$ times larger than that of SR.
\end{lem}
\begin{lem}{\rm (Instance where SR may outperform Almost Tracking)}\label{lem_srwin}
There exists an instance such that the rate of SR is $O(\log K)$ times larger than that of Eq.~\eqref{eq_our_error}.
\end{lem}
Since Theorem \ref{thm:hthree_main} includes the performance lower bound of SR,  Lemma \ref{lem_ourwin} is 
tight; namely, SR's performance is actually worse than Almost Tracking. On the other hand, Lemma \ref{lem_srwin} states that SR outperforms the rate of Eq.~\eqref{eq_our_error}. Since Eq.~\eqref{eq_our_error} may not be tight, this result does not exclude the possibility that our algorithm performs as well as SR for all instances. 

\section{Conclusion}\label{sec:conclusion}
We have considered the fixed budget best arm identification problem. 
We have proposed algorithmic framework that achieves the rate-optimality in view of any given risk measure $H(\bP)$. Based on the framework, we have proposed a closed-form algorithm that has rate-optimality for the standard risk measure $H_1(\bP)$.
An interesting future work is to provide an efficient optimization algorithm to achieve the rate-optimality for a larger class of complexity measures.
Unlike elimination-based methods, our algorithm does not require prior knowledge of $T$ and still outperforms both anytime and fixed-budget baselines, which motivates a redesign of algorithms for many structured pure exploration problems.

\section{For practitioners}\label{sec_practitioners}

Empirically, both Almost Tracking and Simple Tracking have strong performance across synthetic and real-world instances. When one desires the theoretically justified constant-factor guarantees with improved stability in adversarial instances, Almost Tracking is preferable. However, for most applied scenarios we recommend Simple Tracking because it is easy to implement and requires no horizon tuning. Practitioners can therefore start with Simple Tracking for $H_1$. 
For the sake of exposition, we provide an explicit expression of Simple Tracking for $H_1$ in Algorithm \ref{alg:R_track_prac}.

\begin{algorithm}[h]
\caption{Practical Recommendation: Simple Tracking for $H_1$}
\label{alg:R_track_prac}
\begin{algorithmic}[1] 
\STATE Draw each arm once.
\FOR{$t = K+1, K+2, \ldots, T$}
    \STATE Calculate number of draws $N_i(t-1)$ and empirical mean $Q_i(t-1)$ for each arm $i$.
    \STATE Calculate weights 
\[
w_i(\hatbQ(t-1)) := \frac{1}{D_i(\bQ)\ZHone(\bQ)},
\]
where 
\begin{align}
\Delta_i^{Q} &= \max_j Q_j - Q_i \\
H_1'(i, \bQ) &= \sum_{j \ne i} \frac{1}{(\Delta_j^{Q} + \Delta_i^{Q})^2},\\
D_i(\bQ) &= 
\begin{cases} %
(\Delta_i^Q)^2 H_1'(i, \bQ) & (i \notin i^*(\bQ))\\
\min_{i \notin i^*(\bQ)} D_i(\bQ) & (i \in i^*(\bQ))\\
\end{cases}\\
\ZHone(\bQ) &= \sum_{i \in [K]} (1/D_i(\bQ)).
\end{align}
    \STATE Draw arm $I(t) = \argmax_{i\in[K]} \left\{ w_i(\hatbQ(t-1)) - N_i(t-1)/(t-1) \right\}$, where $N_i(t-1)$ be the number of samples for arm $i$ up to round $t-1$.
\ENDFOR
\RETURN $J(T) = $ empirical best arm at the end.
\end{algorithmic}
\end{algorithm}%

\section*{Acknowledgement}
K.~Jang was supported by the Institute of Information \& Communications Technology Planning \& Evaluation (IITP) grant funded by the Korea government (MSIT) [RS-2021-II211341, Artificial Intelligence Graduate School Program (Chung-Ang University)].
J.~Komiyama was supported by the MBZUAI Start-up Fund [BF0121, Machine Learning Department].
J.~Honda was supported by
JSPS KAKENHI (Grant Number JP25K03184).

\clearpage
\bibliography{biblio}
\appendix
\thispagestyle{empty}

\section{Notation}

Table \ref{tab:symbols} summarizes major notation used in this paper.

We use $\gkldiv{Q_i}{P_i}$ to denote the KL divergence between two distributions $Q_i$ and $P_i$. Since we assume the reward distributions are Gaussian, we have
\begin{align}
    \gkldiv{Q_i}{P_i} &= \frac{(P_i - Q_i)^2}{2}.
\end{align}
During the lemmas, we use $\gkldiv{Q_i}{P_i}$ for the results that does not specifically require the Gaussian assumption. 
For lemmas where the Gaussianity is required, we use $\frac{(P_i - Q_i)^2}{2}$ instead of $\gkldiv{Q_i}{P_i}$. 
Appendix \ref{sec:limitations} describes the reward assumptions required for each lemma.

\newcolumntype{L}[1]{>{\raggedright\arraybackslash}p{#1}} %
\renewcommand{\arraystretch}{1.05}  %
\begin{table}[htbp]
  \small                                  
  \caption{Notation table}
  \label{tab:symbols}
  \setlength{\tabcolsep}{4pt}              %
  \begin{tabular}{l p{0.6\linewidth}}  %
    \toprule
    \text{Symbol} & Meaning \\
    \midrule
        $[K] = \{1,2,\dots,K\}$ & Set of arms\\
        $T$ & Total number of samples\\
        $\EP$ & Set of true distributions ($= \Real$ for Gaussian)\\
        $\EQ$ & Set of empirical distributions ($= \Real$ for Gaussian)\\
        $\bP = (P_1,P_2,\dots,P_K) \in \EP^K$ & True means (unknown).\\
        $\bQ(t) = (Q_1(t),Q_2(t),\dots,Q_K(t)) \in \EQ^K$ & Empirical means at time $t$\\
        $N_i(t)$ & \# of pulls of arm $i$ at time $t$\\
        $B$ & \# of batches\\
        $N$ & \# of samples per batch (Algorithm \ref{alg:R_track_almost}). $B=\lfloor T/N \rfloor$.\\
        $\sigma^2 = 1$ & Variance of the reward (known, unit)\\
        $H(\bP)$ & General complexity measure (risk measure)\\
        $H_1(\bP)$ & $\sum_{i \ne i^*} (\max_j P_j - P_i)^{-2}$\\
        $H_2(\bP)$ & $\max_i i \Delta_{(i)}^{-2}$\\
    $H_3(\bP)$ & Characterizing complexity of SR. Defined in Eq.~\eqref{def_h3_a}\\
        $D(Q_i || P_i) = (P_i-Q_i)^2/2$ & Gaussian KL divergence\\
        $i^*(\bP) := \argmax_{i\in[K]} P_i$ & We assume uniqueness for true $\bP$. Namely, $|i^*(\bP)|=1$\\
        $i^*(\bQ) := \argmax_{i\in[K]} Q_i$ & Possibly non-unique for discrete cases (e.g., Bernoulli).\\
        $P^*, Q^*$ & $\max_i P_i$, $\max_i Q_i$\\
        $D_i(\bQ)$ & Defined in Eq.~\eqref{ineq_Di}\\
        $\bQ_b = (Q_{b,1},Q_{b,2},\dots,Q_{b,K})$ & Empirical mean vector at $b$-th batch\\
        $Q_{b,i}$ & Empirical mean of the i-th arm at $b$-th batch\\        
        $N_{b,i}$ & \# of i-th arm pulls in $b$-th batch\\
        $\simplex{K}$ & $K$-dimensional simplex\\
        $\bw(\bQ) = (w_1(\bQ),w_2(\bQ),\dots,w_K(\bQ))$ & Weight function (allocation)\\
        $\bw_b = (w_{b,1},w_{b,2},\dots,w_{b,K})$ & Realized batch weights at batch $b$ on arm $i$\\
        $\Delta_i$ & $P^* - P_i$ \\
        $\Delta_{(i)}$ & $\Delta_j$ of the $i$-th best arm $j$\\
        $\Delta_i^Q$ & $(\max_j Q_j) - Q_i$ \\
    $R_\EA$ & Rate of Algorithm $\EA$ (Definition \ref{def_rate})\\
        $\fracwrap{\bm{x}} = (\fracwrap{x_1},\fracwrap{x_2},\dots,\fracwrap{x_K})$ & Rounding fractional allocation to integer (Appendix \ref{sec:fractional_pulls}).\\
        $\bar{\bQ}_{b-1}$ & $(\bar{Q}_{b,1},\bar{Q}_{b,2},\dots,\bar{Q}_{b,K})$\\
        $\bar{Q}_{b,i}$ & $\frac{\sum_{b'=1}^b w_{b',i} Q_{b',i}}{\sum_{b'=1}^b w_{b',i}}$ (= empirical mean up to batch $b$)\\
        $\olog(K)$ & $\frac{1}{2} + \sum_{i=2}^K \frac{1}{i}$\\
    \bottomrule
  \end{tabular}
\end{table}

\clearpage

\section{Two-approximation algorithm}\label{sec:two_approx}

\begin{algorithm}[t]
    \caption{Pooled Allocation}
    \label{alg:pooled}
    \begin{algorithmic}[1]
    \REQUIRE Number of samples $T$, Number of batches $B$, optimal weight function $\bw^*(\bQ)$, optimal recommendation $i^*(\bQ)$ (any fixed tie-break).
    \STATE Draw each arm for $T/B$ times (equivalently, $w_{b,i} = \mathbf{1}[i=b]$ for $b\in[K]$. Let $Q_{1,i}$ denote the initial empirical mean of arm $i$ (first initialization batch)).
    \STATE Build empirical means $\bm{Q}^{\mathrm{pool}}_{1}$.
    \STATE Initialize the number of victories for each arm: $\bm{v} = (v_1,v_2,\dots,v_K) = (0,0,\dots,0)$.
    \FOR{$b = K+1, K+2, \dots, B$}
    \STATE Increase $v_{i^*(\bm{Q}^{\mathrm{pool}}_{b-K})}$ by one.
        \STATE Calculate optimal allocation based on the pooled means:
        \begin{equation}\label{ineq_allocation}
            \bw_b = \bw^*(\bm{Q}^{\mathrm{pool}}_{b-K})
        \end{equation}
    \STATE Draw each arm for $w_{b,i} \times (T/B)$ times and obtain empirical mean vector $\bm{Q}_b$.
        \STATE Update the delayed means:
        \begin{equation}\label{ineq_update_unused}
            \bm{Q}^{\mathrm{pool}}_{b-K+1} = (\bm{1}-\bm{w}_b) \odot \bm{Q}^{\mathrm{pool}}_{b-K} + \bm{w}_b \odot \bm{Q}_b
        \end{equation}
        where $\odot$ denotes element-wise multiplication.
    \ENDFOR
    \STATE \textbf{return} $J(T) = \arg\max_i v_i$
    \end{algorithmic}
\end{algorithm}
        
This section introduces Pooled Allocation (Algorithm \ref{alg:pooled}), which achieves the two-approximation rate of Theorem \ref{thm_upper}. Since we consider $K,B$ to be sufficiently large and we do not use this algorithm in practice, we assume as if $T/B$ is an integer for simplicity.
\begin{thm}\label{thm_pooled}
Assume that $K^2 = o(T)$.
Algorithm \ref{alg:pooled} achieves the following rate:
\begin{align}
    R_{\mathrm{Pooled Allocation}} \ge 
    \frac{\Cfull}{2}-\varepsilon
\end{align}
for any $\varepsilon > 0$.
\end{thm}
Theorem \ref{thm_pooled} does not use the property of Gaussian KL divergence (i.e., squared form) and holds for a large class of distributions such as an exponential family of distributions. Therefore, we use KL divergence formula $\gkldiv{Q_i}{P_i}$. For the case of Gaussian, $\gkldiv{Q_i}{P_i} = (Q_i-P_i)^2/2$.
\begin{proof}[Proof of Theorem \ref{thm_pooled}]
    By assumption $K^2 = o(T)$, we can take $B$ such that $K = o(B), KB = o(T)$.
    For $\bP$ such that $i^*(\bP) \ne J(T)$, we have
    \begin{align}
    \lefteqn{
    \frac{T}{B} \sum_{b \in [B]} \sum_i w_{b,i} \gkldiv{Q_{b,i}}{P_i}
    }\\
    &\ge 
    \frac{T}{B} \sum_{b \in [B-K]} \sum_i w_i^*(\bQ^{\mathrm{pool}}_{b}) \gkldiv{Q^{\mathrm{pool}}_{b,i}}{P_i}
    \tag{Lemma \ref{lem_convexity} with $B'=B$ and $\bP' = \bP$}\\
    &\ge \frac{T(B-K)}{2B} \frac{R^{\mathrm{go}}}{H(\bm{P})},
\end{align}
where the last inequality follows from the facts that $w^*(\bQ), i^*(\bQ)$ are optimal weight and recommendation, and $i^*(\bP) \ne J(T)$ implies that at least half among $B-K$ batches have $i^*(\bP) \ne i^*(\bQ^{\mathrm{pool}}_{b})$.
By using this, we have
\begin{align}
\Prob[i^*(\bP) \ne J(T)]
&\le \Prob\left[
\frac{T}{B} \sum_{b \in [B]} \sum_i w_{b,i} \gkldiv{Q_{b,i}}{P_i}
\ge 
\frac{T(B-K)}{2B} \frac{R^{\mathrm{go}}}{H(\bm{P})}
\right]\\
&\le \exp\left(-
\frac{T(B-K)}{2B} \frac{R^{\mathrm{go}}}{H(\bm{P})} + o(T)
\right)
\tag{by $KB = o(T)$ and Theorem \ref{lem_abstpoe}}
\end{align}
and the proof is complete by using $K = o(B)$.
\end{proof} %

    \begin{lem}\label{lem_convexity}
    For all $B' \in \{K,K+1,\dots,B\}$ and any $\bP'$, we have the following:
    \begin{equation}\label{ineq_induction}
    \sum_{b \in [B']} \sum_i w_{b,i} \gkldiv{Q_{b,i}}{P_i'}
    \ge 
    \left(
    \sum_{b \in [B'-K]} \sum_i w_{b+K,i} \gkldiv{Q^{\mathrm{pool}}_{b,i}}{P_i'}
    +
    \sum_{i \in [K]} \gkldiv{Q^{\mathrm{pool}}_{B'-K+1,i}}{ P_i'} %
    \right).
    \end{equation}    
    \end{lem}
    \begin{proof}[Proof of Lemma \ref{lem_convexity}]
    We show Eq.\eqref{ineq_induction} by induction for $B'$. It trivially holds for $B=K$.
    Assume that it holds for $B'$. Then 
    \begin{align*}
    \sum_{b \in [B'+1]} &\sum_{i \in [K]} w_{b,i} \gkldiv{Q_{b,i}}{ P_i'}
    \\
    &\ge 
    \sum_{b \in [B'-K]} \sum_{i \in [K]} w_{b+K,i} \gkldiv{Q^{\mathrm{pool}}_{b,i} }{ P_i'}
    +
    \sum_{i \in [K]} \gkldiv{Q^{\mathrm{pool}}_{B'-K+1,i} }{ P_i'}
    \\
    &\ \ \ \ \ \ \ \ + 
    \underbrace{
    \sum_{i \in [K]} w_{B'+1,i} \gkldiv{Q_{B'+1,i} }{ P_i'}
    }_{\text{$B'+1$-th batch}}
    \tag{By the induction hypothesis}
    \\
    &=
    \sum_{b \in [B'-K]} \sum_{i \in [K]} w_{b+K,i} \gkldiv{Q^{\mathrm{pool}}_{b,i} }{ P_i'}
    +
    \sum_{i \in [K]} w_{B'+1,i} \gkldiv{Q^{\mathrm{pool}}_{B'-K+1,i} }{ P_i'}
    \\
    &\ \ \ \ \ \ \ \ 
    +
    \sum_{i \in [K]} (1 - w_{B'+1,i}) \gkldiv{Q^{\mathrm{pool}}_{B'-K+1,i} }{ P_i'}
    + 
    \sum_{i \in [K]} w_{B'+1,i} \gkldiv{Q_{B'+1,i} }{ P_i'}
    \\
    &=
    \sum_{b \in [B'-K+1]} \sum_{i \in [K]} w_{b+K,i} \gkldiv{Q^{\mathrm{pool}}_{b,i} }{ P_i'}
    \\
    &\ \ \ \ \ \ \ \ 
    +
    \sum_{i \in [K]} (1 - w_{B'+1,i}) \gkldiv{Q^{\mathrm{pool}}_{B'-K+1,i} }{ P_i'}
    + 
    \sum_{i \in [K]} w_{B'+1,i} \gkldiv{Q_{B'+1,i} }{ P_i'}
    \\
    &\ge
    \sum_{b \in [B'-K+1]} \sum_{i \in [K]} w_{b+K,i} \gkldiv{Q^{\mathrm{pool}}_{b,i} }{ P_i'}
    +
    \sum_{i \in [K]} \gkldiv{Q^{\mathrm{pool}}_{B'-K+2,i} }{ P_i'},
    \end{align*}
    where the last inequality follows from   Jensen's inequality and $Q^{\mathrm{pool}}_{B'-K+2,i} = w_{b+K,i} Q_{B'+1,i} + (1-w_{b+K,i}) Q^{\mathrm{pool}}_{B'-K+1,i}$.
    Therefore, the inequality holds for $B'+1$.
\end{proof} %

\section{A constant-ratio ceiling}\label{sec:fractional_pulls}

For a practical implementation, the size of $B$ is limited, and thus considering the fractional part on the number of pulls is important.
This section introduces a constant-ratio ceiling function $\fracwrap{\bm{x}} = (\fracwrap{x_1},\fracwrap{x_2},\dots,\fracwrap{x_K})$ that converts a fractional number of pulls to an integer number of pulls. 
This ceiling is used in Algorithm \ref{alg:R_track_almost}.
\begin{algorithm}[t]
    \caption{From fractional number of pulls to integer number of pulls} 
    \label{alg:fracwrap}
    \begin{algorithmic}[1] 
    \STATE Input: Weight function $\bw = (w_1,w_2,\dots,w_K) \in \simplex{K}$, $\Nbatch \in \Natural $.
    \STATE $\Nbatch' = \Nbatch - \sum_{i \in [K]} \Ind[w_i > 0]$
    \FOR{$i \in [K]$}
        \IF{$w_i > 0$}
            \STATE $N_i = 1 + \lfloor w_i \Nbatch' \rfloor$
        \ELSE
            \STATE $N_i = 0$
        \ENDIF
    \ENDFOR
    \WHILE{$\sum_{i \in [K]} N_i < \Nbatch$}
        \STATE Choose $i$ with probability $w_i$ and set $N_i = N_i + 1$.
    \ENDWHILE
    \RETURN $\fracwrap{w_i \Nbatch} = N_i$.
    \end{algorithmic}
\end{algorithm}

\begin{lem}\label{lem:fracwrap}
    Assume that $\Nbatch \ge 2 K$.
    Then, the number of pulls $\fracwrap{w_i \Nbatch}$ computed by Algorithm \ref{alg:fracwrap} satisfies the followings:
    \begin{itemize}
        \item For each $i \in [K]$, $\fracwrap{w_i \Nbatch} \in \Natural$.
        \item $\sum_{i \in [K]} \fracwrap{w_i \Nbatch} = \Nbatch$.
        \item For each $i \in [K]$, $\fracwrap{w_i \Nbatch} \ge w_i \Nbatch / 4$.
    \end{itemize}
\end{lem}
In other words, $\fracwrap{w_i \Nbatch}$ an integer allocation that is at least the constant ($1/4$) times of the optimal allocation for each arm.
\begin{proof}[Proof of Lemma \ref{lem:fracwrap}]
The first two properties are trivial.
For the third property, 
For $i: w_i = 0$, we have $\fracwrap{w_i \Nbatch} = 0 = w_i \Nbatch$.
For $i: w_i \Nbatch \in (0, 4]$, we have $\fracwrap{w_i \Nbatch} \ge 1 \ge \frac{1}{4} w_i \Nbatch$.
For $i: w_i \Nbatch > 4$, we have
\begin{align}
\fracwrap{w_i \Nbatch} 
&\ge \lfloor w_i \Nbatch' \rfloor\\
&\ge w_i \Nbatch' - 1\\
&\ge w_i \Nbatch / 2  - 1 \tag{by $\Nbatch \ge 2K$}
\\
&\ge w_i \Nbatch / 4
\end{align}
and the proof is complete.
\end{proof}

\section{Proof of Theorem \ref{thm_batchtrack}}
\label{sec_batchtrack}

For convenience, we restate the theorem here.
\restatementtrue
\batchtrack* %
\begin{proof}[Proof of Theorem \ref{thm_batchtrack}]

We denote $\bw(\bQ) = (w_1,w_2,\dots,w_K) \in \simplex{K}$ to represent the ideal allocation of one-shot game.

We consider a batched algorithm and let $\bw_{b,i}, Q_{b,i}$ be corresponding allocation and empirical means. 
We denote averaged quality over $B$ batches as:
\begin{align}
\bar{w}_{B,i} 
&= 
\sum_b \frac{w_{b,i}}{B}\\
\bar{Q}_{B,i} 
&= 
\frac{\sum_b w_{b,i} Q_{b,i}}{\sum_b w_{b,i}}.
\end{align}

Let us use a widetilde operator to denote amortized\footnote{Intuitively speaking, we move part of weights during the first $b-1$ rounds to the weight of round $b$. This is purely analytical technique and does not change the algorithmic design.} weights and amortized means as follows:
\begin{align}
\widetilde{w}_{b,i}
&= w_{b,i} + \frac{\Csuf}{(b-1)} \sum_{b'=1}^{b-1} w_{b',i}\\
\widetilde{Q}_{b,i} 
&= \frac{w_{b,i} Q_{b,i} + \frac{\Csuf}{(b-1)} \sum_{b'=1}^{b-1} w_{b',i} \bar{Q}_{b-1,i}}{w_{b,i} + \frac{\Csuf}{(b-1)} \sum_{b'=1}^{b-1} w_{b',i}}.
\end{align}

We consider a sampling rule such that
\begin{equation}\label{eq_samplingone}
w_{b,i} \ge \frac{w_i(\bar{\bQ}_{b-1})}{4}
\end{equation}
holds for all $i$ such that 
\begin{equation}\label{eq_samplingtwo}
\frac{1}{b-1} \sum_{b'=1}^{b-1} w_{b',i} \le \frac{1}{\Csuf} w_i(\bar{\bQ}_{b-1})
\end{equation}
for some $\Csuf \in (0,1)$. 
Here, the factor $1/4$ in Eq.~\eqref{eq_samplingone} is derived from the fractional allocation ($\fracwrap{\cdot}$) and is non-essential.
Algorithm \ref{alg:R_track_almost}, combined with a fractional allocation (Lemma \ref{lem:fracwrap}), belongs to this class of sampling rule.

\eqref{eq_samplingone},\eqref{eq_samplingtwo} imply
\begin{align}\label{eq_samplingmerge}
\widetilde{w}_{b,i}  
&\ge \frac{w_i(\bar{\bQ}_{b-1})}{4}
\end{align}
for any $i$.

Moreover, we have
\begin{align}
\lefteqn{
\sum_{b=1}^{B} \sum_i w_{b,i} (Q_{b,i} - P_i)^2
}\\
&= \sum_i w_{B,i} (Q_{B,i} - P_i)^2
+ 
\frac{\Csuf}{B-1} \sum_{b=1}^{B-1} \sum_i w_{b,i} (Q_{b,i} - P_i)^2
+
\frac{B-1 - \Csuf}{B-1} \sum_{b=1}^{B-1} \sum_i w_{b,i} (Q_{b,i} - P_i)^2\\
&\ge \sum_i \widetilde{w}_{B,i} (\widetilde{Q}_{B,i} - P_i)^2
+
\frac{B-1-\Csuf}{B-1} \sum_{b=1}^{B-1} \sum_i w_{b,i} (Q_{b,i} - P_i)^2
\tag{Jensen's inequality}
\\
\label{ineq_bplustemp}
\end{align}

Repeatedly applying \eqref{ineq_bplustemp} yields%
\begin{align}
\lefteqn{
\sum_{b=1}^{B} w_{b,i} (Q_{b,i} - P_i)^2
}\\
&\ge 
\widetilde{w}_{B,i} (\widetilde{Q}_{B,i} - P_i)^2
+
\frac{B-1-\Csuf}{B-1}\widetilde{w}_{B-1,i} (\widetilde{Q}_{B-1,i} - P_i)^2
+\\
&\ \ \ \ \ \ 
\frac{B-1-\Csuf}{B-1} \frac{B-2-\Csuf}{B-2} \widetilde{w}_{B-2,i}  (\widetilde{Q}_{B-2,i} - P_i)^2
+ \dots\\
&= 
\sum_b C_b \widetilde{w}_{b,i} (\widetilde{Q}_{b,i} - P_i)^2,
\label{ineq_bplustemp_amortized}
\end{align}
where
\begin{equation}\label{ineq_cbdef}
C_b = \prod_{b'=b}^{B-1} \frac{b'-\Csuf}{b'} 
\end{equation}
is a increasing sequence in $b$.\footnote{Note that $C_B = 1$.}
Here, $C_b$ is the weight of the process such that, moving some weight from $\sum_{b'}^{b-1} w_{b'}$ to $w_b$ to build $\widetilde{w}_{b,i}$. This fact implies the weighted average of \updatedduringaistats{$(w_{b',i})_{b'\in[b]}$ and $(C_b' \widetilde{w}_{b',i})_{b' \in [b]}$} is identical, namely, 
\begin{align}
\sum_{b' \le b} w_{b',i}  &= \sum_{b' \le b} C_{b'} \widetilde{w}_{b',i} 
\\
\bar{Q}_{b,i}
&= \frac{
\sum_{b' \le b} C_{b'} \widetilde{w}_{b',i} \widetilde{Q}_{b',i}
}{
\sum_{b' \le b} C_{b'} \widetilde{w}_{b',i} 
}.\label{ineq_barQtildeavg}
\end{align}
Moreover, $\sum_b C_b \ge b(1-\Csuf) = \Omega(B)$.\footnote{Sequence $A_b:$ s.t. $A_1=1$, $A_b = ((b-1-\Csuf)/(b-1)) A_{b-1} + 1$ satisfies $A_b \ge b(1-\Csuf) $.}

\updatedduringaistats{For ease of discussion, we use $1/\Cweightlower$ as if it were an integer.}\footnote{We can easily formalize this by introducing a floor operator.}
The variance term $V(B)$ is 
\begin{align}
V(B) 
&:= \sum_b \sum_i C_b \widetilde{w}_{b,i}
\left(
(\widetilde{Q}_{b,i} - P_i)^2
-
(\bar{Q}_{B,i} - P_i)^2
\right)\\
&\ge \frac{1}{2} \sum_{b=1/\Cweightlower } \sum_i  C_b \widetilde{w}_{b,i} 
(\widetilde{Q}_{b,i} - \bar{Q}_{b-1,i})^2 %
\tag{by Lemma \ref{lem_trans1}}
\\
&\ge \frac{3}{40} \sum_{b=1/\Cweightlower } \sum_i C_b \widetilde{w}_{b,i} (\bar{Q}_{B,i} - \bar{Q}_{b-1,i})^2 
\tag{by Lemma \ref{lem_trans2}}\\
\label{ineq_maindiv_new}
\end{align}
and finally, we have 
\begin{align}
\text{Eq.~\eqref{ineq_bplustemp_amortized}} 
&\ge 
\sum_{b=1/\Cweightlower } \sum_i C_b \widetilde{w}_{b,i} %
(\bar{Q}_{B,i} - P_i)^2
+ V(B)\\
&\ge \frac{3}{40} 
\sum_{b=1/\Cweightlower } C_b \sum_i \widetilde{w}_{b,i}  
\left(
(\bar{Q}_{B,i} - P_i)^2
+  (\bar{Q}_{B,i} - \bar{Q}_{b-1,i})^2
\right)
\tag{by \eqref{ineq_maindiv_new}}
\\ 
&\ge \frac{3}{80}
\sum_{b=1/\Cweightlower } C_b
\sum_i \widetilde{w}_{b,i} (\bar{Q}_{b-1,i} - P_i)^2
\tag{$X^2+Y^2 \ge \frac{1}{2}(X-Y)^2$}\\
&\ge \frac{3}{320}
\sum_{b=1/\Cweightlower } C_b
\sum_i w_i(\bar{\bQ}_{b-1}) (\bar{Q}_{b-1,i} - P_i)^2
\tag{by \eqref{eq_samplingmerge}}\\
&\ge \frac{3}{320}\left(\sum_{b=1/\Cweightlower } C_b\right)
\inf_{\bQ} \sum_i w_i(\bQ) (Q_i - P_i)^2
\end{align}
and the fact $\sum_b C_b = \Omega(B)$ completes the proof of Theorem \ref{thm_batchtrack}.
\end{proof} %

\subsection{Transformation 1}\label{subsec_easytrans}

\begin{lem}{(Transformation 1)}\label{lem_trans1}
For any $i \in [K]$, it holds that
\begin{align}\label{ineq_trans1}
\sum_{b=1}^B C_b \widetilde{w}_{b,i}
\left(
    (\widetilde{Q}_{b,i} - P_i)^2
    -
    (\bar{Q}_{B,i} - P_i)^2
\right)
    \ge \frac{1}{2} \sum_{b=1/\Cweightlower }^B C_b \widetilde{w}_{b,i} 
    (\widetilde{Q}_{b,i} - \bar{Q}_{b-1,i})^2
\end{align}
\end{lem}
\begin{proof}[Proof of Lemma \ref{lem_trans1}]
During the proof, we omit index $i$ for ease of notation. 
For the ease of notation we use $v_b = C_b \widetilde{w}_b$.
Let us represent the LHS minus the RHS of \eqref{ineq_trans1} as
\begin{align}
    X_{b'} &:= \sum_{b=1}^{b'} v_b \left(
    (\widetilde{Q}_b - P)^2 
    -
    (\bar{Q}_{b'} - P)^2 
    \right)
    - \frac{1}{2} \sum_{b=1}^{b'} v_b (\widetilde{Q}_b - \bar{Q}_{b-1})^2
\end{align}
We first show
\begin{equation}\label{ineq_Xbdiff}
X_{b'} - X_{b'-1} \ge 0
\end{equation}
for $b' \ge 1/\Cweightlower +1$.

Letting 
\begin{align}
A&:= v_{b'} \\
B&:= \sum_{b=1}^{b'-1} v_b\\
X&:= \widetilde{Q}_{b'} - P\\
Y&:= \bar{Q}_{b'-1} - P,
\end{align}
we have
\begin{align}
\lefteqn{
X_{b'} - X_{b'-1}
}\\
&= A X^2 
 -  \left( 
    \sum_{b=1}^{b'} v_b (\bar{Q}_{b'} - P)^2 
    -
    \sum_{b=1}^{b'-1} v_b (\bar{Q}_{b'-1} - P)^2 
    \right)
 - \frac{1}{2} A (X-Y)^2 
 \\
 &= A X^2  
 -  \left( 
    (A+B) \left(\frac{AX + BY}{A+B}\right)^2
    -
    B Y^2
    \right)
 - \frac{1}{2} A (X-Y)^2 
 \tag{by \eqref{ineq_barQtildeavg}}
 \\
&= A X^2 + B Y^2 - (A+B) \left(\frac{AX + BY}{A+B}\right)^2 - \frac{1}{2} A (X-Y)^2 \\
&= \frac{1}{A+B} AB (X-Y)^2 - \frac{1}{2} A (X-Y)^2
\label{ineq_ABXY}
\end{align}
Eq.\eqref{ineq_ABXY} is $\ge 0$ if $A \le B$ holds. In the following, we show $A \le B$ for $b'$ is sufficiently large.

First, it is easy to see that
\begin{equation}\label{ineq_cbratio}
\frac{
    \sum_{b=1}^{b'} C_b
}{
    C_{b'+1}
}
\ge b'-1.
\end{equation}
By \eqref{eq_samplingmerge}, we have $\Cweightlower  \le \widetilde{w}_b \le 1$ uniformly. Therefore,
\begin{align}
B &= \sum_{b=1}^{b'-1} C_b \widetilde{w}_b\\
&\ge \Cweightlower \sum_{b=1}^{b'-1} C_b \\
&\ge \Cweightlower (b' - 1) \tag{by \eqref{ineq_cbratio}}\\
&\ge 1 \ge A \tag{if $b' = 1/\Cweightlower +1$}.
\end{align}

In summary,
\begin{align}
X_B 
&= \sum_{b=1}^B (X_b - X_{b-1}) \\
&= \sum_{b=1/\Cweightlower +1}^B (X_b - X_{b-1})
+ X_{1/\Cweightlower }\\
&\ge X_{1/\Cweightlower } \\
\tag{by \eqref{ineq_Xbdiff} $\ge 0$ for $b \ge 1/\Cweightlower +1$}
&= \underbrace{\sum_{b=1}^{1/\Cweightlower } v_b \left(
    (\widetilde{Q}_b - P)^2 
    -
    (\bar{Q}_{\updatedduringaistats{1/\Cweightlower}} - P)^2 
    \right)}_{\ge 0}
    - \frac{1}{2} \sum_{b=1}^{1/\Cweightlower } v_b (\widetilde{Q}_b - \bar{Q}_{b-1})^2\\
&\ge\updatedduringaistats{ - \frac{1}{2} \sum_{b=1}^{1/\Cweightlower } v_b (\widetilde{Q}_b - \bar{Q}_{b-1})^2,}
\end{align}
which is Eq.~\eqref{ineq_trans1}.

\end{proof}

\subsection{Transformation 2}\label{subsec_trans}

We omit index $i$ for ease of notation. For the ease of notation we use $v_b = C_b \widetilde{w}_b$. 

\begin{lem}{(Transformation 2)}\label{lem_trans2}
Let $\Ctranstwo = \frac{20}{3}$. 
For any $c > 1$, it holds that
\begin{equation}
    \Ctranstwo \sum_{b=c}^B v_b (\widetilde{Q}_b - \bar{Q}_{b-1})^2
    -
    \sum_{b=c}^B v_b (\bar{Q}_B - \bar{Q}_{b-1})^2 \ge 0
    \label{ineq_transtwo_goal}
    \end{equation}
\end{lem}
\begin{proof}[Proof of Lemma \ref{lem_trans2}]
Let 
\begin{align}
v_{:b} &= \sum_{b'=c}^{b-1} v_{b'}\\
v_{b:} &= \sum_{b'=b}^B v_{b'}\\
x_b &= \bar{Q}_b - \bar{Q}_{b-1}\\
x_{b:} &= \sum_{b'=b}^B x_{b'}
\end{align}

By using
\[
\frac{\sum_b v_b \widetilde{Q}_b}{\sum_b v_b} = \bar{Q}_b,
\]
Eq.~\eqref{ineq_transtwo_goal} is equivalent to
\begin{align}\label{ineq_vxc}
\Ctranstwo \sum_{b=c}^B v_b \left(\frac{v_b + v_{:b}}{v_b}\right)^2 x_b^2
-
\sum_{b=c}^B v_b x_{b:}^2 
\ge 0.
\end{align}
We will show that this holds for any $(v_b)_{b\in[B]} > 0$, $(x_b)_{b \in [B]} \in \mathbb{R}$. This equation is a quadratic formula for $(x_b)_{b \in [B]}$. 

Namely, letting $\bx = (x_c, x_{c+1},\dots,c_{B})^\top$ be a size $B-c+1$ vector, 
Eq.~\eqref{ineq_vxc} = $\bx^T M \bx$, where
\begin{align*}
    M = \begin{bmatrix}
        D_c & -O_c & -O_c & \cdots & -O_c\\
        -O_c & D_{c+1} & -O_{c+1} & \cdots & -O_{c+1}\\
        -O_c & -O_{c+1} & D_{c+2} & \cdots & -O_{c+2}\\
        \vdots & \vdots & \vdots & \ddots & \vdots\\
        -O_c & -O_{c+1} & -O_{c+2} & \cdots & D_B
    \end{bmatrix} \in \mathbb{R}^{(B-c+1) \times (B-c+1)}
\end{align*}
where $D_i = \Ctranstwo \frac{v_{:i+1}^2}{v_i}-v_{:i+1}$ and $O_i = v_{:i+1}$. 

The problem of showing the positivity of Eq.~\eqref{ineq_vxc} for all $(v_b)_{b\in[B]} > 0$, $(x_b)_{b \in [B]} \in \mathbb{R}$ is equivalent to showing the positive-definiteness of $M$ for any $(v_b)_{b\in[B]}$.
Sylvester's criterion states that, a symmetric matrix $M$ is positive definite if and only if the Gaussian elimination transforms it to a triangular matrix with positive diagonal using only the elementary operation of adding a multiple of the column with another column \citep{kucera2017linear}.  
Namely, we will derive the diagonals of: 
\begin{align*}
    M' = \begin{bmatrix}
        D_c' & 0 & 0 & \cdots & 0\\
        -O_c' & D_{c+1}' & 0 & \cdots & 0\\
        -O_c' & -O_{c+1}' & D_{c+2}' & \cdots & 0\\
        \vdots & \vdots & \vdots & \ddots & \vdots\\
        -O_c' & -O_{c+1}' & -O_{c+2}' & \cdots & D_B'
    \end{bmatrix} 
\end{align*}
The $O_i'$ and $D_i'$ above follows the following recurrence relation:
\begin{align*}
    D_b' &= D_b - Z_{b-1}\\
    O_b' &= O_b + Z_{b-1}
\end{align*}
where $Z_b = \sum_{i=c+1}^{b} \frac{O_i'^2}{D_i'}$. 
\begin{claim}\label{claim_Z} %
Let $b \ge c$ be arbitrary. There exists $\xi_1, \cdots, \xi_b \in (3, \infty]$ such that $Z_b = \sum_{i=1}^b \frac{v_i}{\xi_i}$, where $x/\infty = 0$.
\end{claim} 
\begin{proof} We will prove it by induction. 
    
\paragraph{Base case $b=c$:} $Z_c = 0$ and thus the claim holds.
\paragraph{Induction step:} Suppose that the induction hypothesis holds for all $b<b_0$. For $b=b_0$, 

$$Z_{b_0} = Z_{b_0-1} + \frac{O_b'^2}{D_b'} = \sum_{i=1}^{b-1} \frac{v_i}{\xi_i} + \frac{O_b'^2}{D_b'}$$

Which means, if we prove that $\frac{O_{b_0}'^2}{D_{b_0}'}=\frac{v_{b_0}}{\xi_{b_0}}$ for some $\xi_{b_0} > 3$, then it verifies the induction step.

\begin{align*}
    \frac{O_{b_0}'^2}{D_{b_0}'} &= \frac{\left( O_{b_0} + Z_{b_0-1}\right)^2}{(D_{b_0}-Z_{b_0-1})}\\
    &=\frac{\left( O_{b_0} + \sum_{i=1}^{b_0-1} \frac{v_i}{\xi_i}\right)^2}{(D_{b_0}-\sum_{i=1}^{b_0-1} \frac{v_i}{\xi_i})} \tag{Induction hypothesis}\\
    &=\frac{\left( v_{:b_0+1} + \sum_{i=1}^{b_0-1} \frac{v_i}{\xi_i}\right)^2}{(\Ctranstwo\frac{v_{:b_0+1}^2}{v_{b_0}} - v_{:b_0+1}-\sum_{i=1}^{b_0-1} \frac{v_i}{\xi_i})}\\
    &\leq \frac{\left( 1+\frac{1}{3}
    \right)^2 v_{:b_0+1}^2}{(\Ctranstwo{v_{:b_0+1}^2} - v_{b_0} \cdot v_{:b_0+1}-\frac{1}{3} v_{b_0}\cdot v_{:b_0+1})} \cdot v_{b_0}\tag{$\xi_i >3$ for all $i\leq b_0-1$, $v_{b_0}\geq 0$}\\
    &\leq \frac{\left( 1+\frac{1}{3}
    \right)^2 v_{:b_0+1}^2}{\Ctranstwo{v_{:b_0+1}^2} - (1+\frac{1}{3})v_{:b_0+1}^2} \cdot v_{b_0} \tag{$v_{b_0}<v_{:b_0+1}$}\\
    &=\frac{\frac{16}{9}}{\Ctranstwo-\frac{4}{3}}v_{b_0} = \frac{v_{b_0}}{3} \tag{$\Ctranstwo = \frac{20}{3}$}
\end{align*}
Therefore, since $0\leq \frac{O_{b_0}'^2}{D_{b_0}'} \leq \frac{v_{b_0}}{3}$, there exists a constant $\xi_{b_0}>3$ such that $\frac{O_{b_0}'^2}{D_{b_0}'}=\frac{v_{b_0}}{\xi_{b_0}}$ and therefore one can write $Z_{b_0}= \sum_{i=1}^{b_0} \frac{v_i}{\xi_i}$, which proves the induction step, which derives the claim.
\end{proof}

Now, based on Claim \ref{claim_Z}, one can see $D_b'>0$ for all $b\leq B$ since 
\begin{align*}
    D_b' &= D_b - Z_{b-1}\\
    &= \frac{\Ctranstwo v_{:b+1}^2}{v_b} - v_{:b+1}-\sum_{i=1}^{b-1}\frac{v_i}{\xi_i}\\
    &\geq (\Ctranstwo-1) v_{:b+1}-\frac{1}{3}v_{:b} \tag{$\xi_i \geq 3$ for all $i \in [B]$}\\
    &\geq \left(\Ctranstwo-\frac{4}{3}\right) v_{:b+1} \geq 0
\end{align*}
Therefore, $M'$ has only positive diagonal entries.
By Sylvester's criterion $M$ is a positive definite matrix, which completes the proof.
\end{proof}

\section{Large deviation lemma}\label{sec_thm_constant}

This section introduces the following Lemma, which uses the large deviation bound. We apply this lemma to analyze Almost Tracking (Algorithm \ref{alg:R_track_almost}) and Pooled Allocation (Algorithm \ref{alg:pooled}).
\begin{lem}{\rm (PoE of the batched algorithm, large deviation)}\label{lem_abstpoe}
    Let $x > 0$ be arbitrary. %
    A BAI algorithm is batch-based if it splits $T$ rounds into $B$ predefined batches of consecutive subset of rounds. Let $N_{b,i}$ and $Q_{b,i}$ be corresponding statistics on the $b$-th batch.
    For any batch-based adaptive algorithm, it always holds that
    \begin{equation}\label{ineq_abstpoe}
    \Prob \left[\sum_{b\in [B]}
    \sum_{i \in [K]} N_{b,i} \frac{(Q_{b,i}-P_i)^2}{2}
     \ge x \right]\leq 
    \exp\left( - x +
    \updatedduringaistats{\frac{TK}{N} \max\left\{\log \left( \frac{x N}{TK}, 0\right)\right\} + \frac{TK}{N} (\log N + \Ccounting)} \right),
    \end{equation}
    where $\Ccounting > 0$ is an absolute constant.
\end{lem}

Lemma \ref{lem_abstpoe} does not use the property of Gaussian KL divergence (i.e., squared form) and holds in many cases with large deviation such as exponential family of distributions. Therefore, we use KL divergene formula $\gkldiv{Q_i}{P_i}$.

Since the algorithm is adaptive, $N_{b,i}$ is a random variable depending on the history. To use the concentration inequality, we consider the following equivalent formulation of a batched algorithm.
\begin{itemize}
\item First, the environment (nature) draws $T$ arms for each batch $b \in [B]$ and arm $i \in [K]$, which are independent samples from $P_i$.
\item Second, the batched algorithm runs. Based on the history up to batch $b-1$, it selects the number of samples to draw and receives the first $N_{b,i}$ observations.
\end{itemize}
The main advantage of this procedure is we can deal the samples independently from the algorithm. For example, we can apply Hoeffding inequality to the first $n_{b,i}$ samples of arm $i$ at batch $b$, even though in some cases algorithm does not draw $n_{b,i}$ samples.

\begin{proof}[Proof of Lemma \ref{lem_abstpoe}]
For $s \in \mathbb{N}$, let 
\[
\mathcal{V}_s
:= \left\{
(v_{b,i})_{b\in[B],i\in[K]} \in \mathbb{N}^{B\times K}: \sum_{b,i} v_{b,i} \in 
\left[x+s, x+s+1\right)
\right\},
\]
and $\mathcal{V} := \bigcup_{s = \{0,1,2,\dots\}} \mathcal{V}_s$.
Note that $\mathcal{V}_s$ is a finite set of combination and we can derive the following:
\updatedduringaistats{
\begin{align}
|\mathcal{V}_s| 
&= 
\binom{x+s+BK-1}{BK-1}\\
&\le \binom{x+s+BK}{BK}\\
&\le 
\left( \frac{e (x+s+BK)}{BK} \right)^{BK}
\text{\ \ \ \ (By $\binom{n}{k} \le (en/k)^k$)}
\label{ineq_vssize}
\end{align}
}

\updatedduringaistats{Let $\mathcal{N}(T) = (\{0,1,2,\dots,N\})^{BK}$.}
Let $\bm{n} = \{n_{b,i}\}_{b\in[B],i\in[K]} \in \mathcal{N}(T)$.
\updatedafterneurips{
Let $Q_{b,i}'$ be the empirical mean of the first $n_{b,i}$ samples drawn for arm $i$ in batch $b$.
}
Then, 
\begin{align*}
\mathbb{P}&\left[
\sum_{b\in[B]} \sum_{i \in [K]} N_{b,i} \frac{(Q_{b,i}-P_i)^2}{2} \ge x
\right]\\
&\leq 
\sum_{\bm{n} \in \EN(T)}
\mathbb{P}\left[
\bigcap_{b,i} \{N_{b,i} = n_{b,i}\}, \sum_b \sum_{i} N_{b,i} \frac{(Q_{b,i}-P_i)^2}{2} \ge x\right]\\
&\le 
\sum_{\bm{v} \in \mathcal{V}}
\sum_{\bm{n} \in \EN(T)}
\mathbb{P}\left[
\bigcap_{b,i}\ \left\{
N_{b,i} = n_{b,i}, 
n_{b,i} \gkldiv{Q_{b,i}}{P_i}  \in [v_{b,i}, v_{b,i}+1)
\right\}
\right]\\
&\updatedafterneurips{\le
\sum_{\bm{v} \in \mathcal{V}}
\sum_{\bm{n} \in \EN(T)}
\mathbb{P}\left[
\bigcap_{b,i}\ \left\{
n_{b,i} \gkldiv{Q_{b,i}'}{P_i}  \in [v_{b,i}, v_{b,i}+1)
\right\}
\right]
}
\\
&\updatedafterneurips{=
\sum_{\bm{v} \in \mathcal{V}}
\sum_{\bm{n} \in \EN(T)}
\prod_{b,i}
\mathbb{P}\left[
n_{b,i} \gkldiv{Q_{b,i}'}{P_i}  \in [v_{b,i}, v_{b,i}+1)
\right]\tag{independence}
}\\
&\updatedafterneurips{=
\sum_{s \ge 0} 
\sum_{\bm{v} \in \mathcal{V}_s}
\sum_{\bm{n} \in \EN(T)}
\prod_{b,i}
\mathbb{P}\left[
n_{b,i} \gkldiv{Q_{b,i}'}{P_i}  \in [v_{b,i}, v_{b,i}+1)
\right]}\\
&\le 
\sum_{s \ge 0} 
\sum_{\bm{v} \in \mathcal{V}_s}
\sum_{\bm{n} \in \EN(T)}
\exp\left(- 
\sum_{b,i} v_{b,i}
\right) 
\tag{Hoeffding's inequality}
\\
&\le 
\sum_{s \ge 0} 
\sum_{\bm{v} \in \mathcal{V}_s}
\sum_{\bm{n} \in \EN(T)}
\exp\left(- 
x-s
\right)
\tag{by $\bm{v} \in \mathcal{V}_s$}\\
&\le 
\sum_{s \ge 0} 
\sum_{\bm{v} \in \mathcal{V}_s}
\updatedduringaistats{(N+1)^{KB}}
\exp\left(- 
x-s
\right)\\
&\le 
\sum_{s \ge 0} 
\updatedduringaistats{
e^{BK}\left(\frac{x+s+BK}{BK}\right)^{BK}
}
\updatedduringaistats{(N+1)^{KB}}
\exp\left(- 
x-s
\right)
\tag{Eq.~\eqref{ineq_vssize}}\\
&\updatedduringaistats{
\le 
\left(\frac{e}{BK}\right)^{BK}
(N+1)^{KB}
\sum_{s \ge 0} (x+s+BK)^{BK} \exp(-x-s)
}
\\
&\updatedduringaistats{
\le 
\left(\frac{e}{BK}\right)^{BK}
(N+1)^{KB}
\sum_{s \ge 0} 3^{BK}(x^{BK}+s^{BK}+(BK)^{BK}) \exp(-x-s)
}
\\
&\updatedduringaistats{
\le 
\left(\frac{e}{BK}\right)^{BK}
(3(N+1))^{KB}
\exp(-x)
\sum_{s \ge 0} (x^{BK}+s^{BK}+(BK)^{BK})
\exp(-s)
}
\\
&\updatedduringaistats{
\le 
\left(\frac{e}{BK}\right)^{BK}
(3(N+1))^{KB}
\exp(-x)
O\left(x^{BK}+(BK)^{BK}\right)
}
\\
&\updatedduringaistats{
\le 
e^{BK}
(3(N+1))^{KB}
\exp(-x)
O(\exp(BK\max\{\log (x/BK), 0\}))
}
\\
&\updatedduringaistats{
\le 
\exp\left(- 
x+\updatedduringaistats{BK \max\{\log (x/(BK)),0\} + BK ((\log N) + \Ccounting)}
\right)
}
\end{align*}
\updatedduringaistats{
where $\Ccounting > 0$ is a constant that does not depend on $T,N,K$.
In the final transformation, we have used
$(N+1)^{BK} = \exp( BK \log (N + 1) ), \Ccounting^{BK} = \exp(O(BK))$ for any $\Ccounting>0$, $\sum_s s^{BK} \exp(-s) = O((BK)^{BK})$. By using $N = T/B$, we obtain Eq.~\eqref{ineq_abstpoe}.
}

\updatedduringaistats{
By using Lemma \ref{lem_abstpoe} and simple calculation, we obtain the following lemma:
\begin{lem}\label{lem_abstpoe_choosen}
Let $x = \frac{C \Cfull T}{ H_1(\bm{P})} = O(T/\log K)$.
Assume that $N= \Omega(K (\log K)^3)$. Then,
\begin{equation}\label{ineq_abstpoe_choosen}
\Prob \left[\sum_{b\in [B]}
\sum_{i \in [K]} N_{b,i} \frac{(Q_{b,i}-P_i)^2}{2}
 \ge x \right]\leq 
\exp\left( - C \left(1 - O\left(\frac{1}{\log K}\right)\right) \frac{\Cfull T}{ H_1(\bm{P})} + o(T) \right).
\end{equation}
\end{lem}
\begin{proof}[Proof of Lemma \ref{lem_abstpoe_choosen}]
The exponent of Lemma \ref{lem_abstpoe} is
\begin{align}
- x + \frac{TK}{N} \max\left\{\log \left( \frac{x N}{TK}\right), 0\right\} + \frac{TK}{N} (\log N + \Ccounting)
&= -\frac{C \Cfull T}{ H_1(\bm{P})} + o(T) + \frac{TK}{N} ((\log N) + \Ccounting)\\
&= -\frac{C \Cfull T}{ H_1(\bm{P})} + o(T) + (1 + \Ccounting) O\left( \frac{T}{(\log K)^2}
\right)
\end{align}
and by using $\frac{C \Cfull T}{ H_1(\bm{P})} = O(T/(\log K))$, we have Eq.~\eqref{ineq_abstpoe_choosen}.
\end{proof}
}

\end{proof}

\section{Proof of Theorem \ref{thm:approxopt}}
\label{sec:proof_approxopt}

This section proves the approximate optimality of the allocation for $H_1$.
We restate the theorem here for convenience.
\approxopt*
The core of Theorem \ref{thm:approxopt} is the ``stability lemma'' that we introduce in Appendix \ref{sec:stability}.
\begin{proof}[Proof of Theorem \ref{thm:approxopt}]
Let 
\[
\Stbl(\bQ, \bP) := 
\sum_i \frac{(Q_i - P_i)^2}{D_i(\bQ)} H_1(\bP).
\]
We have
\begin{align}  
\inf_{\bQ\in\EQ^K,\bP: i^*(\bP) \notin i^*(\bQ)} H_1(\bP) \sum_{i\in[K]} w_i(\bQ) D(Q_i\Vert P_i)
&= \inf_{\bQ\in\EQ^K,\bP: i^*(\bP) \notin i^*(\bQ)} \frac{S(\bQ, \bP)}{\ZHone(\bQ)} \\
&\ge \inf_{\bQ\in\EQ^K,\bP: i^*(\bP) \notin i^*(\bQ)}
\frac{\Cstabilityinner}{\inf_{\bQ} \ZHone(\bQ)}
\tag{Lemma \ref{lem_stability}}\\
&\ge \inf_{\bQ\in\EQ^K,\bP: i^*(\bP) \notin i^*(\bQ)}
\frac{\Cstabilityinner}{\log K + 1}
\tag{by Lemma \ref{lem_ratio}}\\
\label{ineq_rateopt_approxopt}
\end{align}
Results in \cite{Carpentier2016} implies the optimal rate is $\Cfull = \Theta(\frac{1}{\log K})$ for $H_1(\bP)$, which matches the rate of Eq.~\eqref{ineq_rateopt_approxopt}.
\end{proof}
    
\subsection{Stability lemma}\label{sec:stability}

\begin{lem}{\rm (Stability)}\label{lem_stability}
    Let $\bP, \bQ \in \Real^K$ be two arbitrary vectors such that $\bP$ has unique best arm $i^*(\bP)$ and $i^*(\bP) \notin i^*(\bQ)$.
    Then, there exists a universal constant $\Cstabilityinner >0$ such that
    $\Stbl(\bQ, \bP) \ge \Cstabilityinner.$ 
\end{lem}    

\begin{proof}[Proof of Lemma \ref{lem_stability}]

\noindent\textbf{Minimize $P$ given $Q$:}
Let
\begin{align}
\Stbl_{2,j}(\bQ) &:= \min_{\bP: i^*(\bP)=j} \Stbl(\bQ, \bP).%
\end{align}
We assume the uniqueness of the best arm in $\bQ$ and $\bP$. The case where the empirical best arm is not unique is discussed in Appendix \ref{subseq_multibest}.
Without loss of generality, assume that $Q_1 > Q_2 \ge \dots \ge Q_K$. Let $\Delta_i^Q = Q_1 - Q_i$. Let $j = \argmax_i P_i$ and $\Delta_i = P_j - P_i$. By definition, $j \ne 1$.

\begin{figure}[htbp]
    \centering
        \includegraphics[width=\textwidth]{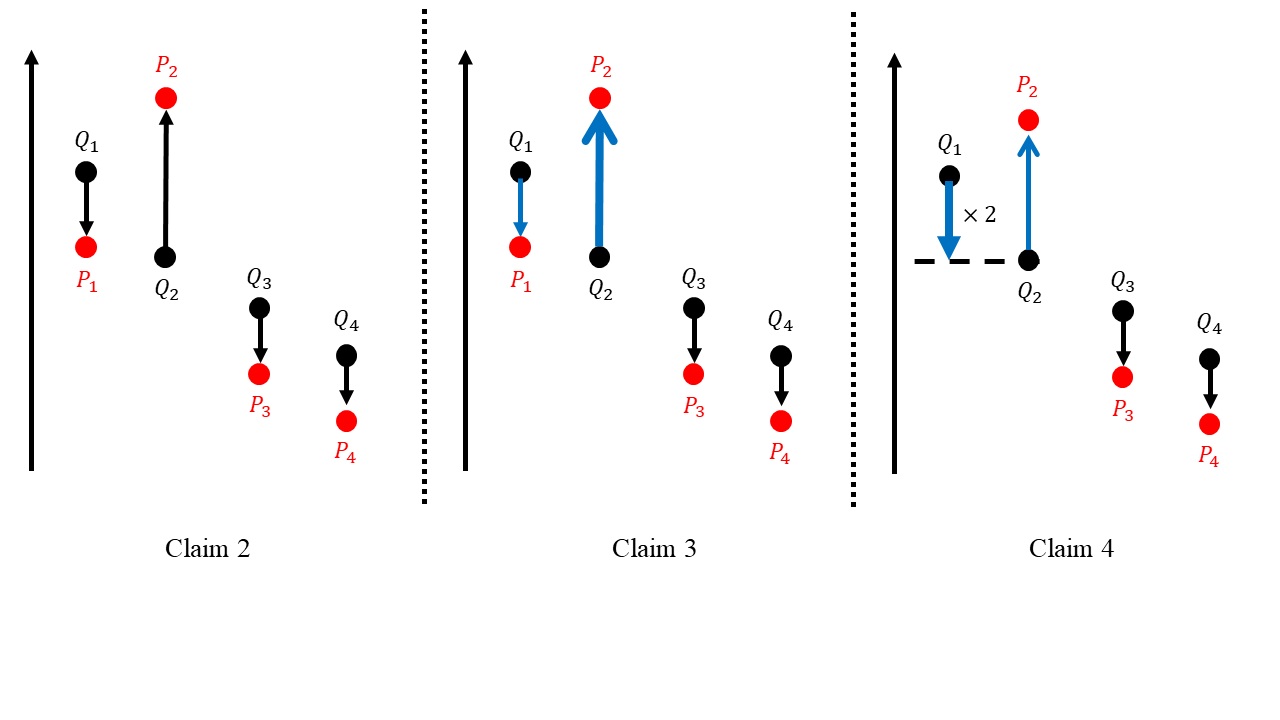}
    \caption{Illustration of the added constraints by Claims 2--4. We use a four-armed instance where the best arm in $\bP$ is arm $j=2$. Claim 2 (left) states that $P_2 > Q_2$ and $P_i \le Q_i$ for $i \ne 2$. Claims 3 (middle) states that bold blue arrow of arm $2$ is longer than the standard-size blue arrow of arm $1$. Claim 4 (right) states that blue arrow of arm $2$ is at most twice as the blue arrow of arm $1$.}
    \label{fig:claims}
\end{figure}

We have
\begin{align*} 
\lefteqn{
\Stbl_{2,j}(\bQ) 
}\\
&= \min_{\bP: i^*(\bP)=j} 
\left(
\sum_i \frac{(Q_i - P_i)^2}{D_i(\bQ)} H_1(\bP)
\right)\\
&= \min_{\bP: i^*(\bP)=j} 
\left(
\sum_i \frac{(Q_i - P_i)^2}{D_i (\bQ)} \sum_{i \ne j} \frac{1}{(P_j - P_i)^2}
\right) \tag{By definition}\\ 
&= \min_{\bP: i^*(\bP)=j, P_j \ge Q_j, \forall_{i \ne j} P_i \le Q_i} 
\left(
\sum_i \frac{(Q_i - P_i)^2}{D_i (\bQ)} \sum_{i \ne j} \frac{1}{(P_j - P_i)^2}
\right) \tag{ Claim \ref{claim: obvious}}\\ 
&= \min_{\bP: i^*(\bP)=j, P_j \ge Q_j, 
\forall_{i < j}\,Q_i - (P_j - Q_j) \le P_i \le Q_i, 
\forall_{i > j}\,P_i \le Q_i
} 
\left(
\sum_i \frac{(Q_i - P_i)^2}{D_i (\bQ)} \sum_{i \ne j} \frac{1}{(P_j - P_i)^2}
\right) \tag{Claim \ref{claim: cprev}}\\
&\geq \frac{1}{4} \min_{\bP: i^*(\bP)=j, Q_j \leq P_j \leq 2Q_1 - Q_j, 
\forall_{i < j}\,Q_i - (P_j - Q_j) \le P_i \le Q_i, 
\forall_{i > j}\,P_i \le Q_i
} 
\left(
\sum_i \frac{(Q_i - P_i)^2}{D_i (\bQ)} \sum_{i \ne j} \frac{1}{(P_j - P_i)^2}
\right) \tag{Claim \ref{claim: pjconst}}\\
&= \Omega(1) \tag{Claim \ref{claim: final lower bound of V2q}}
\end{align*}
and the proof of Lemma \ref{lem_stability} is completed.
\end{proof}

\subsection{Proof of Claims}

\begin{claim}\label{claim: obvious}
    $\min\limits_{\bP: i^*(\bP)=j} \Stbl(\bQ, \bP) = \min\limits_{\bP: i^*(\bP)=j, P_j \ge Q_j, \forall_{i \ne j} P_i \le Q_i} \Stbl(\bQ, \bP)$
\end{claim}
\begin{proof}
    For a $\bP \in \{\bP: i^*(\bP)=j, P_j < Q_j\}$, define $\bP'$ as $$P_s'=\begin{cases}
        Q_j & s=j\\
        P_s & \text{Otherwise}
    \end{cases}.
$$ Then, one can check $i^*(\bP')=j$ and $\Stbl(\bQ, \bP)>\Stbl(\bQ,\bP')$. Therefore, we can assume that the minimum argument should satisfy $P_j\geq Q_j$. Similarly, we can prove that the optimal argument satisfies $P_i \leq Q_i$ for all $i\neq j$. 
\end{proof}

\begin{claim}\label{claim: cprev}
    $$\min\limits_{\bP: i^*(\bP)=j, P_j \ge Q_j, \forall_{i \ne j} P_i \le Q_i} \Stbl(\bQ, \bP) = \min\limits_{\bP: i^*(\bP)=j, P_j \ge Q_j, 
\forall_{i < j}\,Q_i - (P_j - Q_j) \le P_i \le Q_i, 
\forall_{i > j}\,P_i \le Q_i
}  \Stbl(\bQ, \bP)$$
\end{claim}

\begin{proof}
    For $\bP \in \{\bP: i^*(\bP)=j, P_j \ge Q_j, \forall_{i \ne j} P_i \le Q_i\}$ suppose that there exists a coordinate $t<j$ such that $Q_t -P_t \geq P_j - Q_j$. Then, define $P'$ as $$P_s'=\begin{cases}
        Q_j + \frac{Q_t-P_t +P_j-Q_j}{2} & s=j\\
        Q_t - \frac{Q_t-P_t +P_j-Q_j}{2} & s=t\\
        P_s & \text{Otherwise}
    \end{cases}$$. Then, since $Q_t -P_t \geq P_j - Q_j$, $P_j' > Q_j + P_j-Q_j= P_j$. In addition, $P_j' - P_t' =  P_j - P_t$. This implies $H_1(\bP')<H_1(\bP)$. Now let's compare $\sum_i \frac{(Q_i - P_i)^2}{D_i(\bQ)}$ and $\sum_i \frac{(Q_i - P_i')^2}{D_i(\bQ)}$. 
    \begin{align*}
        \sum_i \frac{(Q_i - P_i)^2}{D_i(\bQ)}-\sum_i \frac{(Q_i - P_i')^2}{D_i(\bQ)}&=\frac{(Q_j - P_j)^2 - (Q_j-P_j')^2}{D_j(\bQ)}+\frac{(Q_t-P_t)^2 - (Q_t - P_t')^2}{D_t(\bQ)}\\
        &=\frac{(P_j' - P_j)(2Q_j-P_j-P_j')}{D_j(\bQ)}+\frac{(P_t'-P_t)(2Q_t-P_t-P_t')}{D_t(\bQ)}\\
        &=\frac{\frac{Q_t -P_t +Q_j - P_j}{2} \cdot \frac{3Q_j - 3P_j -Q_t + P_t}{2}}{D_j(\bQ)}+\frac{\frac{Q_t -P_t +Q_j-P_j}{2}\cdot\frac{3Q_t -3P_t +P_j -Q_j}{2}}{D_t(\bQ)}\\
        &\geq\frac{\frac{Q_t -P_t +Q_j - P_j}{2} \cdot \frac{3Q_j - 3P_j -Q_t + P_t}{2}}{D_j(\bQ)}+\frac{\frac{Q_t -P_t +Q_j-P_j}{2}\cdot\frac{3Q_t -3P_t +P_j -Q_j}{2}}{D_j(\bQ)} \tag{$D_j(\bQ)>D_t(\bQ)$ and numerator is positive}\\
        &\geq\frac{\frac{Q_t -P_t +Q_j - P_j}{2} \cdot \frac{2(Q_j - P_j) -2(Q_t - P_t)}{2}}{D_j(\bQ)}>0, \tag{by assumption $Q_t -P_t \geq P_j - Q_j$}
    \end{align*}
    and thus $\sum_i \frac{(Q_i - P_i)^2}{D_i(\bQ)}>\sum_i \frac{(Q_i - P_i')^2}{D_i(\bQ)}$. Combined with $H_1(\bP') < H_1(\bP)$, we have $\Stbl(\bQ,\bP')<\Stbl(\bQ, \bP)$. 
\end{proof}

\begin{claim}\label{claim: pjconst}
\begin{align*}
   &\min\limits_{\bP: i^*(\bP)=j, Q_j \leq P_j, 
\forall_{i < j}\,Q_i - (P_j - Q_j) \le P_i \le Q_i, 
\forall_{i > j}\,P_i \le Q_i
}\Stbl(\bQ,\bP) \\\geq &\frac{1}{4} \min\limits_{\bP: i^*(\bP)=j, Q_j \leq P_j \leq 2Q_1 - Q_j, 
\forall_{i < j}\,Q_i - (P_j - Q_j) \le P_i \le Q_i, 
\forall_{i > j}\,P_i \le Q_i
} \Stbl(\bQ,\bP)\end{align*}  
\end{claim}

\begin{proof}
    Let $\Stbl_{2,j}(\bP ; \bQ):= \sum_i \frac{(Q_i - P_i)^2}{D_i(\bQ)} H_1(\bP)$, and let $\bP^0:= \arg \min_{\bP: i^* (\bP)=j} \Stbl_{2,j}(\bP)$. If $\bP_j^0\leq 2Q_1 - Q_j$, then the Claim holds trivially. Suppose that $\bP_j^0> 2Q_1 - Q_j$, and let $E:=\bP_j^0-Q_1 > Q_1 -Q_j$. Now, let's define $\bP'$ as
    $$P_s'=\begin{cases}
        Q_1 + \Delta_j^Q & s=j\\
        Q_s - (Q_s-P_s)\cdot \frac{\Delta_j^Q}{E}& \text{Otherwise}
    \end{cases}$$

Let's check how this operation changes the value $\Stbl_{2,j}(\bP,\bQ)$, especially for each $W_{2,j}^1 (\bP;\bQ):=  \sum_i \frac{(Q_i -P_i)^2}{D_i(\bQ)}$ and $W_{2,j}^2 (\bP;\bQ):=  {\sum_i \frac{1}{(P_j - P_i)^2}}$ (Note that $\Stbl_{2,j}(\bP;\bQ)=W_{2,j}^1 (\bP;\bQ) \cdot W_{2,j}^2 (\bP;\bQ))$. 

\paragraph{First term $W_{2,j}^1 (\bP;\bQ)$} Since $E>\Delta_j^Q$, $Q_s - P_s' < Q_s - P_s$ for all $s\neq j$, it holds that $W_{2,j}^1 (\bP';\bQ)=\sum_i \frac{(Q_i -P_i')^2}{D_i(Q)} \leq \sum_i \frac{(Q_i -P_i^0)^2 }{D_i(Q)}=W_{2,j}^1 (\bP^0;\bQ)$. Namely, this operation decreases $W_{2,j}^1$. 
In particular, for each $i \neq j$, $Q_s - P_s' = (Q_s - P_s)\cdot \frac{\Delta_j^Q}{E}$, and $P_j' - Q_j = 2\Delta_j^Q = \frac{2\Delta_j^Q}{E+\Delta_j^Q} \cdot (P_j^0-Q_j) \leq \frac{2\Delta_j^Q}{E} \cdot (P_j^0-Q_j)$. Therefore, $W_{2,j}^1 (\bP';\bQ) \leq \frac{(\Delta_j^Q)^2}{E^2}\left(\sum_{i\neq j} \frac{P_i^0-Q_i}{D_i(\bQ)}\right) + \frac{4(\Delta_j^Q)^2}{E^2} \frac{P_j^0-Q_j}{D_j (\bQ)}< \frac{4(\Delta_j^Q)^2}{E^2} W_{2,j}^1 (\bP^0;\bQ)$.

\paragraph{Second term $W_{2,j}^2 (\bP;\bQ)$} By Claim \ref{claim: obvious}, $P_j^0>P_j' > Q_1 > P_s'>P_s^0$ for all $s\neq j$ so $W_{2,j}^2 (\bP^0;\bQ)=\sum_{i \neq j} \frac{1}{(P_j^0 - P_i^0)^2} \leq \sum_{i\neq j}\frac{1}{(P_j' - P_i')^2}=W_{2,j}^2 (\bP';\bQ)$,
which means this operation increase the function $W_{2,j}^2$. More specifically, for each $i \neq j$, 
\begin{align*}
    {(P_j' - P_i')} &= {Q_1 +E \cdot \frac{\Delta_j^Q}{E} - Q_s + (Q_s-P_s^0)\cdot \frac{\Delta_j^Q}{E}}\\
    &= {E \cdot \frac{\Delta_j^Q}{E} +\Delta_s^Q + (Q_s-P_s^0)\cdot \frac{\Delta_j^Q}{E}} \\
    &\geq {E \cdot \frac{\Delta_j^Q}{E} +\Delta_s^Q \frac{\Delta_j^Q}{E} + (Q_s-P_s^0)\cdot \frac{\Delta_j^Q}{E}} \tag{$\frac{\Delta_j^Q}{E}\leq 1$}\\
    &= \frac{\Delta_j^Q}{E} \cdot \left(E+\Delta_s^Q + Q_s-P_s^0\right) = \frac{\Delta_j^Q}{E} \cdot \left(P_j^0 - P_s^0\right). \tag{$E=P_j^0-Q_1, \Delta_s^Q +Q_s = Q_1$}
\end{align*}

Therefore, we can conclude $W_{2,j}^2 (\bP';\bQ) \leq W_{2,j}^2 (\bP^0;\bQ) \cdot \frac{E^2}{(\Delta_j^Q)^2}$.

Overall, $\Stbl_{2,j}(\bQ) = W_{2,j}^1 (\bP^0;\bQ) \cdot W_{2,j}^2 (\bP^0;\bQ) \geq \frac{E^2}{4(\Delta_j^Q)^2} W_{2,j}^1(\bP';\bQ) \cdot \frac{(\Delta_j^Q)^2}{E^2} W_{2,j}^2 (\bP';\bQ) = \frac{1}{4} \Stbl_{2,j}(\bP';\bQ)$. Also, one can check that $P'$ satisfies the constraint of the second optimization. 
Therefore, in {both cases} of $\bP^0$, the claim holds. 
\end{proof}

\begin{claim} \label{claim: final lower bound of V2q}
    $\min\limits_{\bP: i^*(\bP)=j, Q_j \leq P_j \leq 2Q_1 - Q_j, 
\forall_{i < j}\,Q_i - (P_j - Q_j) \le P_i \le Q_i, 
\forall_{i > j}\,P_i \le Q_i
} \Stbl(\bQ, \bP)= \Omega(1)$
\end{claim}

\begin{proof}
    Let's focus on $W_{2,j}^1(\bP;\bQ)$ first.
    \begin{align*}
        W_{2,j}^1 (\bP;\bQ)&=\sum_l \frac{(Q_l -P_l)^2}{D_l(\bQ)}\\
        &\geq \sum_{l\geq j} \frac{(Q_l -P_l)^2}{D_l(\bQ)} \\
        &\geq \frac{(\Delta_j^Q)^2}{D_j(\bQ)} + \sum_{l>j}\frac{(P_l - Q_l)^2}{D_l(\bQ)} \tag{for $l=j$, we can use $P_j \ge Q_1$}
    \end{align*}
    Next, we can lower bound $W_{2,j}^2 (\bP;\bQ)$ as follows:
    \begin{align*}
        W_{2,j}^2(\bP;\bQ)&=\sum_{i \neq j} \frac{1}{(P_i - P_j)^2}\\
        &= \sum_{i < j} \frac{1}{(P_i - P_j)^2} + \sum_{i>j} \frac{1}{(P_i - P_j)^2}\\
        &\geq \sum_{i<j} \frac{1}{9(\Delta_j^Q)^2} + \sum_{i>j} \frac{1}{(P_i - P_j)^2} \tag{for $i<j$, $P_j - P_i = P_j - Q_1 + Q_1 -Q_i + Q_i- P_i\leq 3\Delta_j$}\\
        &\geq \frac{j}{18(\Delta_j^Q)^2} + \sum_{i>j} \frac{1}{(P_i - P_j)^2} \tag{$j\geq 2 \rightarrow j-1\geq j/2$}
    \end{align*}
    Now, let $M_i := |Q_i - P_i|$ for notational convenience. 
    \begin{align*}
        \lefteqn{
        \Stbl_{2,j}(\bP;\bQ)
        }\\&= W_{2,j}^1 (\bP;\bQ) \cdot W_{2,j}^2(\bP;\bQ)\\
        &\geq \paren{\frac{(\Delta_j^Q)^2}{D_j(\bQ)} + \sum_{l>j}\frac{(P_l - Q_l)^2}{D_l(\bQ)}}\cdot \paren{\frac{j}{18(\Delta_j^Q)^2} + \sum_{l>j} \frac{1}{(P_l - P_j)^2}} \\
        &\geq \frac{1}{D_j(\bQ)}\paren{1 + \sum_{l>j} \frac{M_l^2}{(\Delta_l^Q)^2}}\cdot \paren{\frac{j}{18} + \sum_{l>j} \frac{(\Delta_j^Q)^2}{(M_j +Q_j -Q_l + M_l )^2}} \tag{Lemma \ref{lem_asdf}}\\
        &\geq \frac{1}{D_j(\bQ)}\paren{1 + \sum_{l>j} \frac{M_l^2}{(\Delta_l^Q)^2}}\cdot \paren{\frac{j}{18} + \sum_{l>j} \frac{(\Delta_j^Q)^2}{(\Delta_j^Q + \Delta_l^Q+ M_l )^2}} \tag{$M_j \leq 2\Delta_j$}\\
        &= \frac{1}{D_j(\bQ)}\paren{1 + \sum_{l>j} \frac{M_l^2}{(\Delta_l^Q)^2}}\cdot \paren{\frac{j}{18} + \sum_{l>j, M_l \geq \Delta_l} \frac{(\Delta_j^Q)^2}{(\Delta_j^Q + \Delta_l^Q+ M_l )^2}+\sum_{l>j, M_l < \Delta_l} \frac{(\Delta_j^Q)^2}{(\Delta_j^Q + \Delta_l^Q+ M_l )^2}} \\
        &\geq \frac{1}{D_j(\bQ)}\paren{1 + \sum_{l>j} \frac{M_l^2}{(\Delta_l^Q)^2}}\cdot \paren{\frac{j}{18} + \sum_{l>j, M_l \geq \Delta_l} \frac{(\Delta_j^Q)^2}{9M_l^2}+\sum_{l>j, M_l < \Delta_l} \frac{(\Delta_j^Q)^2}{9(\Delta_l^Q)^2}} 
        \tag{by $j > l$ implies $\Delta_j^Q \le \Delta_l^Q$}
        \\
        &\geq  \frac{1}{D_j(\bQ)}\paren{1 + \sum_{l>j} \frac{M_l^2}{(\Delta_l^Q)^2}}\cdot \paren{\frac{j}{18} + \sum_{l>j, M_l \geq \Delta_l} \frac{(\Delta_j^Q)^2}{9M_l^2}} +\frac{1}{D_j(\bQ)}\sum_{l>j, M_l < \Delta_l} \frac{(\Delta_j^Q)^2}{9(\Delta_l^Q)^2}\\
        &\geq \frac{1}{D_j(\bQ)}\paren{1 + \sum_{l>j, M_l \geq \Delta_l} \frac{M_l^2}{(\Delta_l^Q)^2}}\cdot \paren{\frac{j}{18} + \sum_{l>j, M_l \geq \Delta_l} \frac{(\Delta_j^Q)^2}{9M_l^2}} +\frac{1}{D_j(\bQ)}\sum_{l>j, M_l < \Delta_l} \frac{(\Delta_j^Q)^2}{9(\Delta_l^Q)^2}\\
        &\geq \frac{1}{D_j(\bQ)} \paren{\sqrt{\frac{j}{18}}+ \sum_{l>j, M_l \geq \Delta_l}\frac{\Delta_j^Q}{3\Delta_l^Q}}^2+\frac{1}{D_j(\bQ)}\sum_{l>j, M_l < \Delta_l} \frac{(\Delta_j^Q)^2}{9(\Delta_l^Q)^2}\tag{Cauchy}\\
        &\geq \frac{1}{D_j(\bQ)} \paren{\frac{j}{18}+ \sum_{l>j}\frac{(\Delta_j^Q)^2}{9(\Delta_l^Q)^2}}\tag{$(a+b)^2>a^2+b^2$ when $a,b>0$}\\
        &=\Omega(1) \tag{From Lemma \ref{lem_dlower} and \ref{lem_dupper}, $D_j(\bQ)= \Theta\paren{j+ \sum_{l>j}\frac{(\Delta_j^Q)^2}{(\Delta_l^Q)^2}}$}
    \end{align*}
\end{proof}

\subsection{Other lemmas}

\begin{lem}
    For $s<j$, $D_s(\bP)<D_j(\bP)$. For $u>j$, $D_u(\bP)<\frac{\Delta_u^2}{\Delta_j^2} D_j(\bP)$.
\end{lem}
\begin{proof}
    Now for $s<j$, 
        $\Delta_s<\Delta_j$, and therefore
        \begin{align*} 
   D_j (\bP) - D_s (\bP)
    &= (\Delta_j + x)^2\sum_{z \ne j} \frac{1}{(\Delta_z + x)^2}-(\Delta_s + x)^2 \sum_{z \ne s} \frac{1}{(\Delta_z + x)^2}\\
    &\geq (\Delta_j + x)^2\sum_{z \ne j} \frac{1}{(\Delta_z + x)^2}-(\Delta_j + x)^2 \sum_{z \ne s} \frac{1}{(\Delta_z + x)^2}\\
    &= (\Delta_j + x)^2\left(\frac{1}{(\Delta_s + x)^2}-\frac{1}{(\Delta_j + x)^2}\right)>0
    \end{align*}
    Now for $u>j$, 
        $\Delta_u>\Delta_j$, and therefore
        \begin{align*} 
    \frac{2\Delta_u^2}{\Delta_j^2}D_j (\bP) - D_u (\bP) &=     \frac{2\Delta_u^2}{\Delta_j^2}(\Delta_j + x)^2\sum_{z \ne j} \frac{1}{(\Delta_z + x)^2}-(\Delta_u + x)^2 \sum_{z \ne u} \frac{1}{(\Delta_z + x)^2}\\
    &\geq     2(\Delta_u + x)^2\sum_{z \ne j} \frac{1}{(\Delta_z + x)^2}-(\Delta_u + x)^2 \sum_{z \ne u} \frac{1}{(\Delta_z + x)^2}\\
    &> (\Delta_u + x)^2\left(\frac{1}{(\Delta
_j + x)^2}-\frac{2}{(\Delta_u + x)^2} + \frac{1}{x^2} \right)>0
    \end{align*}
\end{proof}

\begin{lem}\label{lem_asdf}
    For $l>j$, $\frac{D_j(\bQ)}{D_l(\bQ)}\geq \frac{1}{4}\frac{\Delta_j^2}{\Delta_l^2}$
\end{lem}
\begin{proof}
    First, note that for $a,b,c,d>0$, $\frac{a+c}{b+d} \geq \min \paren{\frac{a}{b}, \frac{c}{d}}$. To see this, WLOG $\frac{a}{b}>\frac{c}{d}$, which means $ad>bc$. Then
    \begin{align}\label{lemeqn: partial fraction}
        \frac{(a+c)}{b+d}-\frac{c}{d} &= \frac{ad+cd - bc-cd}{(b+d)d}>0
    \end{align}
    Now, from Lemma \ref{lem_dlower} and \ref{lem_dupper}, $D_j(\bQ) \geq j + \sum_{s>j} \frac{\Delta_j^2}{\Delta_s^2}$ and $D_l(\bQ) \leq 4(l + \sum_{s>l} \frac{\Delta_l^2}{\Delta_s^2})$. Therefore,
    \begin{align*}
        \frac{D_j(\bQ)}{D_l(\bQ)}&\geq \frac{1}{4} \cdot \frac{j + \sum_{s>j} \frac{\Delta_j^2}{\Delta_s^2}}{l + \sum_{s>l} \frac{\Delta_l^2}{\Delta_s^2}}\\
        &= \frac{1}{4} \cdot \frac{\overbrace{\paren{j + \sum_{s: j<s\leq l} \frac{\Delta_j^2}{\Delta_s^2}}}^{A} + \overbrace{\paren{\sum_{s>l} \frac{\Delta_j^2}{\Delta_s^2}}}^{C}}{\underbrace{l}_{B} +\underbrace{ \sum_{s>l} \frac{\Delta_l^2}{\Delta_s^2}}_{D}}\\
        &\geq \min\paren{\frac{\paren{j + \sum_{s: j<s\leq l} \frac{\Delta_j^2}{\Delta_s^2}}}{l}, \frac{\sum_{s>l} \frac{\Delta_j^2}{\Delta_s^2}}{\sum_{s>l} \frac{\Delta_l^2}{\Delta_s^2}}} \tag{Eq. \eqref{lemeqn: partial fraction}}\\
        &\geq \min\paren{\frac{\paren{j + (l-j)\cdot \frac{\Delta_j^2}{\Delta_l^2}}}{l}, \frac{\Delta_j^2}{\Delta_l^2}}\geq \frac{\Delta_j^2}{\Delta_l^2} \tag{$\Delta_s \leq \Delta_l$ for $s\leq l$}
    \end{align*}
    and the proof ends.
\end{proof}

\begin{lem}\label{lem_dlower}
Suppose that $P_1 \ge P_2 \ge ... \ge P_K$. Then
$D_i(\bP) \ge i + \sum_{j > i} \frac{\Delta_i^2}{\Delta_j^2}$ .
\end{lem}

\begin{proof}
\begin{align*}
D_i(\bP)
\ge \inf_{x > 0} (\Delta_i+x)^2 \sum_{j \ne i} \frac{1}{(\Delta_j + x)^2}
&\ge \sum_{j \ne i} \inf_{x > 0} \frac{(\Delta_i+x)^2}{(\Delta_j + x)^2}
\\
&\ge \sum_{j \ne i} \min\{1, \frac{\Delta_i^2}{\Delta_j^2}\}
\\
&= i + \sum_{j > i} \frac{\Delta_i^2}{\Delta_j^2}.
\end{align*}
\end{proof}

\begin{lem}\label{lem_dupper}
Suppose that $P_1 \ge P_2 \ge ... \ge P_K$. Then
$D_i(\bP) \le 4\left(i + \sum_{j > i} \frac{\Delta_i^2}{\Delta_j^2}\right)$ .
\end{lem}

\begin{proof}
\begin{align*}
D_i(\bP)
&{=} (\Delta_i+\Delta_i)^2 \sum_{j \ne i} \frac{1}{(\Delta_j + \Delta_i)^2}\\
&= {4 \cdot \left(\sum_{j>i} \frac{\Delta_i^2}{(\Delta_i+\Delta_j)^2}+\sum_{j<i} \frac{\Delta_i^2}{(\Delta_i+\Delta_j)^2}\right)
{\le}
4 \cdot \left(\sum_{j>i} \frac{\Delta_i^2}{\Delta_j^2}+i\right)}\\
\: .
\end{align*}
\end{proof}

\begin{lem}\label{lem_ratio}
For any $\bP$, $Z(\bP) = \Omega(1)$ and $Z(\bP) = O(\log K)$.
\end{lem}
\begin{proof}
Lemma \ref{lem_dupper} that implies $1/D_i \ge 1/(4K)$, from which $Z(\bP) = \Omega(1)$ directly follows.

By Lemma \ref{lem_dlower}, we have
\begin{align}
Z(\bP) 
:= \sum_i \frac{1}{D_i}
\le \sum_i \frac{1}{i + \sum_{j > i} \frac{\Delta_i^2}{\Delta_j^2}}
\le \sum_i \frac{1}{i} \le \log(K)+1.
\end{align}
\end{proof}

\subsection{Proof of Lemma \ref{lem_stability} for multiple empirical best arms}
\label{subseq_multibest}

Let $\bQ$ be any empirical means with multiple best arm $|i^*(\bQ)| \ge 2$. Let $\bP: i^*(\bP) \in i^*(\bQ)$ be any true distribution with unique best arm (Note: if $i^*(\bP) \notin i^*(\bQ)$, then we can just remove all best arm ties in $i^*(\bQ)$ and just keep one, which is trivial).
Our goal here is to show
\begin{equation}\label{eq:multibestgoal}
\sum_i \frac{(Q_i - P_i)^2}{D_i(\bQ)} H(\bP) \ge C \sum_i \frac{(Q_i' - P_i)^2}{D_i(\bQ)} H(\bP)
\end{equation}
for some universal constant $C>0$.
Where $\bQ', \bP$ are distributions with unique best arm and $i^*(\bQ') \ne i^*(\bP)$.

Without loss of generality, we assume $Q_1 = Q_2 \ge Q_3 \ge \dots \ge Q_K$ and $P_1 > \{P_2, P_3, \dots, P_K\}$. Let $j > 2$ be the first index such that $Q_j < Q_1$.

Moreover, let $\bQ'$ be the distribution such that
\begin{align}
    Q_i' = \begin{cases}
    Q_i & i \ne 2\\ 
    Q_2 + \varepsilon & i = 2 %
    \end{cases},
\end{align}
for a small enough constant $\epsilon>0$. Then $i^*(\bQ') = 2$ is the unique best arm and for any $i$, we have
\[
(Q_i - P_i)^2 - (Q_i' - P_i)^2 = O(\varepsilon) 
\]
Moreover, by using Lemma \ref{lem_dlower} and Lemma \ref{lem_dupper}, we have the followings.
Namely, for any $i\ge j$, we have 
\begin{align}
D_i(\bQ) 
\le
4 \left( j + \sum_{l>j} \frac{(\Delta_i^Q)^2}{(\Delta_l^Q)^2} \right)
\end{align}
and
\begin{align}
    D_i(\bQ') 
    \ge
    j + \sum_{l>j} \frac{(\Delta_i^Q)^2}{(\Delta_l^Q  + \varepsilon)^2}.
\end{align}
For any $i < j$, we have 
\begin{align}
    D_i(\bQ) = D_j(\bQ)
    \le
    4 \left( j - 1 + \sum_{l \ge j} \frac{(\Delta_j^Q)^2}{(\Delta_l^Q)^2} \right)
\end{align}
and 
\begin{align}
    D_i(\bQ') = D_{j-1}(\bQ')
    \ge
    j - 1 + \sum_{l \ge j} \frac{(\Delta_j^Q)^2}{(\Delta_l^Q  + \varepsilon)^2} 
\end{align}

which together imply
\begin{equation}
    \sum_i \frac{(Q_i - P_i)^2}{D_i(\bQ)} H(\bP) \ge \frac{1}{4} \sum_i \frac{(Q_i' - P_i)^2}{D_i(\bQ')} H(\bP) - O(\varepsilon)
\end{equation}
        
Since $1 = i^*(\bP) \ne i^*(\bQ') = 2$, we obtained the desired $\bQ', \bP$ such that Eq.~\eqref{eq:multibestgoal} holds for $C > 1/4 - \varepsilon$ with any $\varepsilon > 0$.

\section{Tighter analysis of SR}
\label{sec:tight_sr_analysis}

\subsection{Upper bound on the probability of error of SR}
\label{subsec:junya_sr_lower_bound}

In this section, we show the rate of SR derived in \cite{Audibert10} is tight up to a constant factor. The next section shows the exact constant factor.

We assume that $P_1>P_2\ge\dots\ge P_K$ and consider Gaussian rewards with unit variance.
The original SR paper \citep{Audibert10} considered rewards over $[0,1]$.
If we run the same argument as the SR paper for unit-variance Gaussian rewards instead of rewards over $[0,1]$, 
we obtain
\begin{align}\label{eq:sr_upper_bound}
\limsup_{T\to\infty}\frac{1}{T}\log \perr(\bP)\le -\frac{1}{H_2(\bP)\olog K}
\end{align}
for
\begin{align}
\olog(K) &= \frac{1}{2} + \sum_{i=2}^K \frac{1}{i}\\
H_2(\bP) &= 4\max_{j\in[K]}j\Delta_j^{-2}
\end{align}
where there is an extra factor of 4.
This is because the Hoeffding inequality $\Prob[Q_{i,n}\ge P_i+x]\le \e^{-2n x^2}\;(x\ge 0)$ %
needs to be replaced with the Chernoff bound $\Prob[Q_{i,n}\ge P_i+x]\le \e^{-n x^2/2}$. %

Note that this Eq.~\eqref{eq:sr_upper_bound} is only an upper bound on the probability of error of SR, and there remains the possibility that the probability of error of SR is smaller. In the next section, we show that this rate SR is tight up to a factor of 2. 

\subsection{Tight bound on the probability of error of SR}
\label{subsec:tight_sr_lower_bound}

We first introduce the modified mean of top-$j$ arms and a complexity $H_3(\bP)$ for the lower bound of the probability of error. We then show that $H_3$ is tight up to a factor of 2.

Define a subset of top-$j$ arms consisting of the best arm and other arms with means at most $P$ as
\begin{align}
\iset{j}(P)=\{1\}\cup \{i\in\{2,3,\dots,j\}: P_i \le P\}.
\end{align}

The \textit{modified mean} of top-$j$ arms is defined by
\begin{align}
\modP{j}=\frac{1}{|\iset{j}(\modP{j})|}\sum_{i\in\iset{j}(\modP{j})}P_i.
\end{align}
Though this expression is written in an implicit form, we can easily see that
$\modP{k}$ satisfying this expression uniquely exists and it satisfies $\modP{j}\in [P_j,P_1]$.

Define 
\begin{align}
\Delta_i^{(j)}=
\begin{cases}
P_1-\modP{j}&i=1\\
(\modP{j}-P_i)_+&i=2,3,\dots,j
\end{cases}
\end{align}
and
\begin{align}\label{def_h3_a}
H_3(\bP)^{-1}
&=\min_{j\in\{2,3,\dots,K\}} \frac{1}{2j}\sum_{i=1}^j (\Delta_i^{(j)})^2,
\end{align}
that is,
\begin{align}
H_3(\bP)
&=\max_{j\in\{2,3,\dots,K\}} \frac{2j}{\sum_{i=1}^j (\Delta_i^{(j)})^2}.\label{def_h3_b}
\end{align}
We will show that $H_3$ is the tight exponent of the probability of error of SR.
\begin{thm}[Restatement of Theorem \ref{thm:hthree_main}]\label{thm:hthree}
    The error probability of successive reject for Gaussian arms satisfies
    \begin{align}
    \lim_{T\to\infty}\frac{1}{T}\log \perr=-\frac{1}{H_3(\bP)\olog K}.
    \end{align}
\end{thm}

Before proving Theorem \ref{thm:hthree}, we consider the relation between $H_2$ and $H_3$.
\begin{thm}\label{thm:h2h3}
It holds that
\begin{align}
\frac{1}{2}< \frac{H_3(\bP)}{H_2(\bP)}\le 1.
\end{align}
In addition, there exists a sequence of instances that the left inequality becomes arbitrarily close to the equality,
and there exists another instance such that the right inequality holds with equality.
\end{thm}
From this result, we see that the existing exponent bound $H_2(\bP)\olog K$ (i.e., Eq.~\eqref{eq:sr_upper_bound}) is tight up to a factor of 2.

Recall that there is an instance such that the exponent of our algorithm is better by a factor of $O(\log \log K)$ compared with $O(H_2 \log K)$.
From this result, the better exponent for this instance does not come from the looseness of the existing error probability upper bound for SR,
but rather comes from essentially better error probability for this instance.
In other words, our algorithm is admissible with respect to SR.

\begin{proof}[Proof of Theorem \ref{thm:h2h3}]
First we derive the upper bound of $H_3$.
\begin{align}
H_3(\bP)^{-1}
&=\min_k \frac{1}{2j}\sum_{i=1}^j (\Delta_i^{(j)})^2
\nn
&\ge \min_k \frac{1}{2j}\left(
(\Delta_1^{(j)})^2
+(\Delta_j^{(j)})^2
\right)
\nn
&\ge
\min_k \frac{1}{2j}\left(
2\cdot \left(\frac{\Delta_j}{2}\right)^2
\right)
\nn
&\ge
\min_k \frac{1}{4j}
(\Delta_j)^2=H_2^{-1}.
\end{align}
The inequalities become equalities when $K=2$.

Next consider the lower bound.
From the definition of $\modP{j}$, we see that
\begin{align}
\sum_{i=2}^j \Delta_i^{(j)}=\Delta_1^{(j)}.
\end{align}
In addition, $\Delta_i^{(j)}$ trivially satisfies $\Delta_i^{(j)}\ge 0$ and $\Delta_1^{(j)}+\Delta_i^{(j)}\le \Delta_j$.
By Lagrange multiplier method,
under these constraints
\begin{align}
\sum_{i=1}^j (\Delta_i^{(j)})^2
\end{align}
is maximized when
$\Delta_2=\Delta_3=\dots=\Delta_j$, where $\modP{j}=P_1- (1-1/j)\Delta_j$.
Then
\begin{align}
\sum_{i=1}^j (\Delta_i^{(j)})^2
&\le
\left(\frac{(j-1)\Delta}{j}\right)^2
+
(j-1)\left(\Delta-\frac{(j-1)\Delta}{j}\right)^2
\nn
&=
(1-1/j)\Delta_j^2.
\nn
&<
\Delta_j^2.
\end{align}
Then we have
\begin{align}
H_3(\bP)^{-1}
&\le \min_j \frac{1}{2j}\sum_{i=1}^j (\Delta_i^{(j)})^2
\nn
&< \min_j \frac{\Delta_j^2}{2j}
\nn
&=2H_2(\bP)^{-1},
\end{align}
where inequalities become arbitrarily close to equalities
when $\Delta_2=\Delta_3=\dots=\Delta_K=\Delta$ with $\Delta>0$ and $K\to\infty$.
This completes the proof of Theorem \ref{thm:h2h3}.
\end{proof}

\begin{proof}[Proof of Theorem \ref{thm:hthree}]
First we prove the error probability upper bound.

The best arm is rejected at the $k$-th phase only when the means of at least $K-k$ suboptimal arms including the already rejected ones
exceeded the mean of the best arm.
Then
\begin{align}
\perr
&\le
\sum_{k=1}^{K-1}\sum_{S\subset [K] \setminus \{1\}: |S|=K-k} %
\Prob\left[\prod_{j\in S}\{\hat{X}_{j,k}\ge \hat{X}_{1,n_k}\}\right]
\nn
&\le
\sum_{k=1}^{K-1}\sum_{S\subset [K] \setminus \{1\}: |S|=K-k} %
\Prob\left[\prod_{j=2}^{K-k+1}\{\hat{X}_{j,k}\ge \hat{X}_{1,n_k}\}\right]
\nn
&=
\sum_{k=1}^{K-1}
\binom{K-1}{K-k}
\Prob\left[\prod_{j=2}^{K-k+1}\{\hat{X}_{j,k}\ge \hat{X}_{1,n_k}\}\right]
\nn
&=
\sum_{k=1}^{K-1}
\binom{K-1}{K-k}
\Prob\left[(\hat{X}_{1,n_k},\dots,\hat{X}_{K-k+1, n_k})\in S_k\right].
\end{align}
where $S_k=\{x\in \mathbb{R}^{K-k+1}: x_j \ge x_1,\,\forall j\in\{2,3,\dots,K-k+1\}\}$.
Since $S_k$ is a convex set, by Cram\'er theorem we have
\begin{align}
\lefteqn{
\Pr\left[(\hat{X}_{1,n_k},\dots,\hat{X}_{K-k+1})\in S_k\right]
}\nn
&\le
\exp\left(
-n_k \inf_{x \in S_k}\sup_{\lambda\in\mathbb{R}^{K-k+1}}\left\{\lambda^{\top}x-\log \E[\e^{\lambda^\top X}]\right\}
\right)\nn
&\le
\exp\left(
-n_k \inf_{x \in S_k}\sum_{j=1}^K \frac{(x-P_i)^2}{2}
\right),
\end{align}
for $X=(X_1,X_2,\dots,X_{K-k+1})$ with $X_i$ independently following $N(P_i, 1)$. 
We can easily see that the infimum is attained at $x_1=\modP{K-k+1}$ and $x_j=\max\{\modP{K-k+1}, P_j\}$, which results in
\begin{align}
\Pr\left[(\hat{X}_{1,n_k},\dots,\hat{X}_{K-k+1})\in S_k\right]
\le
\exp\left(
-\frac{n_k}{2} \sum_{j=1}^{K-k+1}(\Delta_j^{(K-k+1)})^2
\right).
\end{align}
Finally we obtain
\begin{align}
\perr
&\le
\sum_{k=1}^{K-1}
\binom{K-1}{K-k}
\exp\left(
-\frac{T-K}{2(K+1-k)\olog K} \sum_{j=1}^{K-k+1}(\Delta_j^{(K-k+1)})^2
\right)
\nn
&\le
\left(
\sum_{k=1}^{K-1}
\binom{K-1}{K-k}
\right)
\exp\left(
-
\frac{T-K}{2\olog K}
\inf_{k\in[K-1]}
\frac{1}{K+1-k}
\sum_{j=1}^{K-k+1}(\Delta_j^{(K-k+1)})^2
\right)\nn
&\le
\left(
\sum_{k=1}^{K-1}
\binom{K-1}{K-k}
\right)
\exp\left(
-
\frac{T-K}{2\olog K}
\inf_{k\in\{2,3,\dots,K\}}
\frac{1}{k}
\sum_{j=1}^{k}(\Delta_1^{(k)})^2
\right)\nn
&\le
2^K
\exp\left(
-
\frac{T-K}{H_3\olog K}
\right).
\end{align}
Then we immediately obtain the asymptotic upper bound
\begin{align}
\lim_{T\to\infty}\frac{1}{T}\log \perr\le -\frac{1}{H_3\olog K}.
\end{align}

Next we consider the lower bound.
Though we can obtain an explicit error probability lower bound when we consider Gaussians,
we only derive an asymptotic bound so that it can be easily extended to general distributions.
Let $Y_{i,k}$ be the mean of the samples from arm $i$ at the $k$-th phase, and $m_k=n_k-n_{k-1}$ be the number of samples
of each arm at the $k$-th phase.
Then we have
\begin{align}
\hat{X}_{i,n_k}=\frac{\sum_{j=1}^k m_jY_{i,j}}{n_k}.
\end{align}
Take optimal $j^*$ achieving the maximum in \eqref{def_h3_a} and define $k^*=K-j^*+1$.
Consider event $\mathcal{E}$ that
for all $k=1,2,\dots,k^*$,
\begin{align}
&Y_{1,k}\in (P_{j^*+1}, \modP{j^*}],\nn
&Y_{j,k}\in (\modP{j^*}, \infty),\,j\in\{2,3,\dots,j^*\},\nn
&Y_{j,k}\in (-\infty, P_{j^*+1}],\,j\in\{j^*+1,\dots,K\}.
\end{align}
Under this event, arm 1 is not rejected for the first $k^*-1$ phases
and is rejected at the $k^*$-th phase.
Then we have $\perr\ge \Pr[\mathcal{E}]$.
Here, by Cram\'er's theorem we have
\begin{align}
\lim_{m_k\to\infty}\frac{1}{m_k}\log \Pr[Y_{1,k}\in (P_{j^*+1}, \modP{j^*}]]
&=
-\frac{(\Delta_1^{(j^*)})^2}{2}
\nn
\lim_{m_k\to\infty}\frac{1}{m_k}\log \Pr[Y_{j,k}\in (\modP{j^*}, \infty)]
&=
-\frac{(\Delta_j^{(j^*)})^2}{2},\,j\in\{2,3,\dots,j^*\}
\nn
\lim_{m_k\to\infty}\frac{1}{m_k}\log \Pr[Y_{j,k}\in (-\infty, \modP{j^*+1}]]
&=
0,\,j\in\{j^*+1,\dots,K\}.
\end{align}
Then we have
\begin{align}
\lim_{T\to\infty}\frac{1}{T}\log \perr
&\ge
\lim_{T\to\infty}\frac{1}{T}\log \Pr[\mathcal{E}]
\nn
&=
-\lim_{T\to\infty}\sum_{k=1}^{k^*}\frac{m_k}{T}\sum_{j=1}^{j^*}\frac{(\Delta_j^{(j^*)})^2}{2}
\nn
&=
-\lim_{T\to\infty}\frac{n_{k^*}}{T}\sum_{j=1}^{j^*}\frac{(\Delta_j^{(j^*)})^2}{2}
\nn
&=
-\frac{1}{(K+1-k^*)\olog K}\sum_{j=1}^{j^*}\frac{(\Delta_j^{(j^*)})^2}{2}
\nn
&=
-\frac{1}{j^*\olog K}\sum_{j=1}^{j^*}\frac{(\Delta_j^{(j^*)})^2}{2}
\nn
&=
-\frac{1}{H_3(\bP)\olog K}.
\end{align}
This completes the proof of Theorem \ref{thm:hthree}.
\end{proof}
 
\section{Proof of Lemma \ref{lem_ourwin}}
\label{sec_ourwin}

The goal of this section is to show an example where the proposed algorithm outperforms SR.
The structure of this section is as follows:
Appendix \ref{subsec:outwin} shows the instance where the proposed algorithm's rate is larger than SR's known rate by a factor of $\Omega(\log K/(\log\log K))$. 
However, this does not exclude the possibility that SR's known rate $H_2$ is loose (Appendix \ref{subsec:junya_sr_lower_bound}). To see our algorithm provably outperforms SR, we show the SR's rate based on $H_2$ is tight up to a factor of 2 (Appendix \ref{subsec:tight_sr_lower_bound}).

\subsection{Explicit construction of the instance}
\label{subsec:outwin}

Let $\bP$ be such that
\begin{align*}
P_1 &= 2,\\
P_2 &= P_3 = \dots = P_{\log K} = 1,\\
P_{\log K +1} &= \dots = P_K =0.
\end{align*}
Then, on this instance, SR's rate\footnote{Here, an algorithm has rate $r$ if its probability of error is $\mathrm{poly}(T)\exp(-rT)$ asymptotically.} is proportional to
\begin{align}\label{ineq:rate_outwin_sr}
\frac{1}{(\log K) H_2(\bP)} 
\propto \frac{1}{(\log K) K}.
\end{align}
Meanwhile, Theorem \ref{thm_batchtrack} implies the rate of Almost Tracking is lower bounded by
\begin{align}
    \inf_{\bQ: i^*(\bP) \notin i^*(\bQ)} 
    \sum_i w_i(\bQ) \frac{(Q_i-P_i)^2}{2} 
    &\propto
    \frac{1}{H_1(\bP)}
    \inf_{\bQ: i^*(\bP) \notin i^*(\bQ)} \frac{\Stbl(\bQ, \bP)}{Z(\bQ)}\\
    &\propto \frac{1}{K}\inf_{\bQ: i^*(\bP) \notin i^*(\bQ)} \frac{\Stbl(\bQ, \bP)}{Z(\bQ)}
    \label{ineq:rate_outwin}
\end{align}

\begin{lem}\label{lem:rate_outwin}
For any $\bQ$ such that $i^*(\bP) \notin i^*(\bQ)$, it holds that
\begin{equation}
    \frac{\Stbl(\bQ, \bP)}{Z(\bQ)} =  \Omega\left( \frac{1}{(\log\log K)} \right).
\end{equation}
\end{lem}
Lemma \ref{lem:rate_outwin}, combined with Eq.~\eqref{ineq:rate_outwin_sr} and Eq.~\eqref{ineq:rate_outwin}, implies that our algorithm outperforms SR by a factor of $\Omega(\log K/(\log\log K))$.

\begin{proof}[Proof of Lemma \ref{lem:rate_outwin}]
    \begin{align}
    \frac{\Stbl(\bQ, \bP)}{Z(\bQ)}
    &:= \frac{\sum_i \frac{(Q_i-P_i)^2}{D_i(\bQ)}}{\sum_i \frac{1}{D_i(\bQ)}}
    H_1(\bP)\\
    &\ge
    \frac{\sum_i \frac{(Q_i-P_i)^2}{D_i(\bQ)}}{\sum_i \frac{1}{D_i(\bQ)}} \times K\\
    &\ge
    \frac{1}{4} \frac{\sum_i (Q_i-P_i)^2}{\sum_i \frac{1}{D_i(\bQ)}}
    \tag{Lemma \ref{lem_dupper} implies $1/D_i(\bQ) \ge 1/(4K)$}\\
    &\ge
    \frac{1}{4} \frac{\sum_i (Q_i-P_i)^2}{\sum_i \frac{1}{i +  \sum_{j>i} \frac{(\Delta_i^Q)^2}{(\Delta_j^Q)^2}}},\tag{by Lemma \ref{lem_dlower}}\\
    &\label{ineq_approx_vz}
    \end{align}
    where $\Delta_i^Q := Q_1 - Q_i$.
    We say arm $i$ \textit{moved significantly} if $|P_i - Q_i| \ge \updatedduringaistats{1/3}$.
    Eq.\eqref{ineq_approx_vz} implies that 
    \[
    \frac{\Stbl(\bQ, \bP)}{Z(\bQ)}
    \ge \Omega\left(
    \frac{\text{\# of significant moves}}{\log K}
    \right).
    \]
    If number of significant moves are $\ge \log K$, then $\frac{\Stbl(\bQ, \bP)}{Z(\bQ)} = \Omega(1)$. In the following, we assume there are at most $(\log K) - 1$ significant moves, which implies $\max_i Q_i \ge \updatedduringaistats{2/3}$.
    
    Let $i, j > \log{K}$. 
    If both $i,j$ did not move significantly, then $Q_i, Q_j \in \updatedduringaistats{[-1/3, 1/3]}$ and thus 
    \begin{align} 
    \frac{(\Delta_i^Q)^2}{(\Delta_j^Q)^2} 
    \ge \left( \frac{2/3 - 1/3}{2/3 - (- 1/3)} \right)^2 \ge 1/9 = \Omega(1).
    \label{ineq_nosigmove}
    \end{align}
    By assumption, the number of significantly moved arms is at most $\log K$. 
    Therefore, 
    \begin{align}
    \text{(Denominator of Eq.~\eqref{ineq_approx_vz})}
    &= \sum_{i\in[K]} \frac{1}{i +  \sum_{j>i} \frac{(\Delta_i^Q)^2}{(\Delta_j^Q)^2}}
    \\
    &\le \sum_{i \le \log K} \frac{1}{i} %
    + \sum_{i \in [\log K + 1, 2 \log K]} \frac{1}{i}
    +
    K \times \frac{1}{(K-\log K)/9} \\
    &\le 
     \sum_{i \le 2 \log K} \frac{1}{i}
    + O(1)\\
    &= O(\log(2 \log K) + 1) = O(\log\log K),\label{ineq_dlimit}
    \end{align}
    where, the second term is derived by the following discussion: 
    At most $\log K$ arm among those in $\{\log K+1, \dots, K\}$ are significantly moved. If $i$ has not significantly moved, then at least $K - \log K$ of $\{j: (\Delta_i^Q)^2/(\Delta_j^Q)^2\}$ are $\ge 1/9$ due to Eq.~\eqref{ineq_nosigmove} and the fact that at most $\log K$ of $j$ has significant move.
    
    In summary, 
    \begin{align}
    \eqref{ineq_approx_vz} 
    &\ge \frac{1}{\sum_i \frac{1}{i +  \sum_{j>i} \frac{(\Delta_i^Q)^2}{(\Delta_j^Q)^2}}}\tag{$\argmax_i Q_i \ne 1$ requires at least one significant move, which implies $\sum_i (Q_i-P_i)^2 = \Omega(1)$}\\
    &\ge \Omega\left(\frac{1}{\log\log K}\right) \tag{by Eq.~\eqref{ineq_dlimit}}.
    \end{align}
\end{proof} %

\section{Proof of Lemma \ref{lem_srwin}}
\label{sec_srwin}

This section shows an instance where SR outperforms our algorithm's bound.
The particular instance is such that $P_i = - \sqrt{i}$ for $i=1,2,\dots,K$ where $i^*(\bP) = 1$. In this case, the rate of SR is
\begin{align}\label{ineq:rate_srwin_sr}
\frac{1}{(\log K) H_2(\bP)} 
    \propto \frac{1}{(\log K)}.
\end{align}
For Almost Tracking, $H_1(\bP)=\Theta(\log K)$ and the rate is 
\begin{align}\label{ineq:rate_srwin_ours}
\frac{1}{H_1(\bP)}\inf_{\bQ: i^*(\bP) \notin i^*(\bQ)} \frac{\Stbl(\bQ, \bP)}{Z(\bQ)}
&:= \inf_{\bQ: i^*(\bP) \notin i^*(\bQ)} \frac{1}{Z(\bQ)} \sum_i \frac{(Q_i - P_i)^2}{D_i(\bQ)}\\
&= O\left( \frac{1}{(\log K)^2} \right)
\end{align}
where the last step is derived from the instance $\bQ$ such that
\[
Q_i = \begin{cases}
    P_i & \text{if } i \ne 2 \\
    2 P_1 - P_2 & \text{if } i = 2.
\end{cases}
\]
Here, $D_i = \Theta(i + \sum_{j>i} \frac{\Delta_i^2}{\Delta_j^2})$ (Eq.~\eqref{lem_dlower} and Eq.~\eqref{lem_dupper}), from which we have
\begin{align}
D_1(\bQ) &= D_2(\bQ) = \Theta(\log K)\\
Z(\bQ) &= \sum_i \frac{1}{D_i(\bQ)} = \Theta\left(\sum_j \frac{1}{\log K + j}\right) = \Theta(\log K).
\end{align}

In summary, comparing Eq.~\eqref{ineq:rate_srwin_sr} and Eq.~\eqref{ineq:rate_srwin_ours}, SR outperforms the bound of our algorithm by the rate of $\Omega(\log K)$. 
Note that Eq.~\eqref{ineq:rate_srwin_ours} only bounds our rate from below and the actual rate of our algorithm might be better than this.

\section{Non-convexity of the objective}\label{sec_nonconvexity}

\updatedafterneurips{
Eq.~\eqref{def_rgo} is an optimization of $\bw(\cdot) = \bw(\bQ)$:
\[
\sup_{\bw(\cdot) \in \simplex{K}}\,
\inf_{\bQ \in \EQ^K}\, \inf_{\bP: \Ist(\bP) \notin \Ist(\bQ)} H(\bP)
\sum_{i\in[K]} w_i(\bQ) D(Q_i\Vert P_i).
\]
The optimization can be solved by maximizing
\begin{equation}\label{ineq_eachw}    
V_{\bQ}(\bw) 
:= 
\inf_{\bP: \Ist(\bP) \notin \Ist(\bQ)} H(\bP)
\sum_{i\in[K]} w_i D(Q_i\Vert P_i)
\end{equation}
for each $\bw = \bw(\bQ)$. Eq.~\eqref{ineq_eachw}, which is a infimum of a linear objective, is concave in $\bw$. 
However, obtaining the value of $V_{\bQ}(\bw)$ for a given vector $\bw$ is a non-convex problem of finding the infimum of $\bP$ for typical $H(\bP)$. 
}

\section{Technical limitations}\label{sec:limitations}

In this paper, we have considered the best arm identification problem with sub-Gaussian rewards. While sub-Gaussian rewards are common in practice and cover major distributions, including any bounded distributions as well as Gaussian distributions, it is still important to consider the case where the rewards are not sub-Gaussian. 
Moreover, while the results in Section \ref{sec_trackability} hold universally for any risk measure $H(\bP)$, the results in Section \ref{sec_optimization} build upon the most widely adopted risk measure $H_1(\bP)$. %
\begin{itemize}
\item The two-approximation algorithm (Appendix \ref{sec:two_approx}) is quite general. It should hold with a large class of one-parameter distributions (e.g., one-parameter exponential family) and any risk measure $H(\bP)$.
\item The constant-ratio ceiling (Appendix \ref{sec:fractional_pulls}) should hold with any distribution and any risk measure.
\item Appendix \ref{sec_batchtrack}: Theorem \ref{thm_batchtrack}, which guarantees the trackability of Almost Tracking, depends on the sub-Gaussian assumption. 
\item Appendix \ref{sec_thm_constant}: Lemma \ref{lem_abstpoe}, which is about large deviation, is easily extendable for a larger class of distributions, such as one-parameter exponential family of distributions.
\item Appendix \ref{sec:proof_approxopt}: Theorem \ref{thm:approxopt}, or the stability Lemma \ref{lem_stability}, which is used by this theorem, depends on the sub-Gaussian assumption and property of the risk measure $H_1(\bP)$. 
\item Appendix \ref{sec:tight_sr_analysis}: Theoretical results for SR (Appendix \ref{sec:tight_sr_analysis}) depend on the sub-Gaussian assumption, just like the existing results of \cite{Audibert10}.
\item Appendix \ref{sec_ourwin} and Appendix \ref{sec_srwin}: The instances of Lemma \ref{lem_ourwin} and Lemma \ref{lem_srwin} are for sub-Gaussian rewards since we compare SR and Almost Tracking.
\end{itemize}

\section{Details of the experiments}

Our code is implemented in Python and runs on a standard Linux server. 
Our code does not require any GPUs nor does it require a large amount of memory.

\subsection{Compared algorithms}\label{subsec:compared_algorithms}

\begin{itemize}
    \item Successive rejects (SR, \cite{Audibert10}): This algorithm splits the rounds into $K-1$ batches of predefined sizes. At the end of each batch, it eliminates one arm. This algorithm requires $T$.
    \item Sequential halving (SH, \cite{Karnin2013}): This algorithm splits the rounds into predefined size of $\lceil \log_2 K \rceil$ batches. At the end of each batch, it eliminates half of the arms. This algorithm requires $T$.
    \item Continuous rejects (CR, \cite{wang2023best}): This algorithm is a dynamic version of SR where the size of each batch is adaptive. This algorithm requires $T$. There are two versions of CR. In our experiments, we show the results of CR-A, which consistently outperforms CR-C. \updatedafterneurips{We set $\theta_0 = 0.01$.}
    \item Double sequential halving (DSH, \cite{DBLP:conf/icml/ZhaoSSJ23}): This is a meta-algorithm that implements a doubling epoch strategy that combined with SH. It begins with $O(K \log_2 K)$ samples and doubles the sample budget in each subsequent epoch, running a complete SH procedure within each epoch.
    This algorithm does not require $T$ beforehand.
    \item Double successive rejects (DSR, a natural doubling-based anytime variant of SR): This algorithm is an version of DSH where the base SH algorithm is replaced by SR.
    \item EB-TC and TS-TC \citep{shang2020fixed,jourdan2022top} are two empirically good versions of top-two Thompson sampling \citep{Russo2020} that are designed for the fixed-confidence identification.
    \item \updatedduringaistats{EB-TC$_{\varepsilon_0}$ \citep{jourdaneps2023} is a version of EB-TC with $\varepsilon_0$-best identification. The value of $\varepsilon_0$ is set to $0.10$. We chose the ``fixed'' version of it with $\beta=0.5$. These choices are based on their report that setting $\varepsilon_0=0$ suffers from poor empirical performance for moderate value of $\delta$ as well as Figures 5, 6, and 7 therein. }
    \item Uniform is a naive algorithm that draws arms in a round-robin fashion.
    \item Simple tracking: Algorithm \ref{alg:R_track} in this paper. This algorithm does not require $T$ beforehand.
    \item Almost tracking: Algorithm \ref{alg:R_track_almost} in this paper. This algorithm does not require $T$ beforehand. We set the size of each batch $N$ to be $2K$ and  $\Csuf=0.999$ for all our experiments.
\end{itemize}

\subsection{Instances used to calculate Table \ref{tab:minimax_rates}}
\label{subsec:instances_minimax_rates}

This section shows the instances that we used to calculate the minimax rates in Table \ref{tab:minimax_rates}. For all instances, we set $K=40$. The budget $T$ is determined so that the $\PoE$ of the best algorithm is between 0.01 and 0.1. We run each simulation for \Runnum{} times and the figures are empirical PoE values.

\begin{figure}[htbp]
    \centering
   \includegraphics[width=0.8\textwidth]{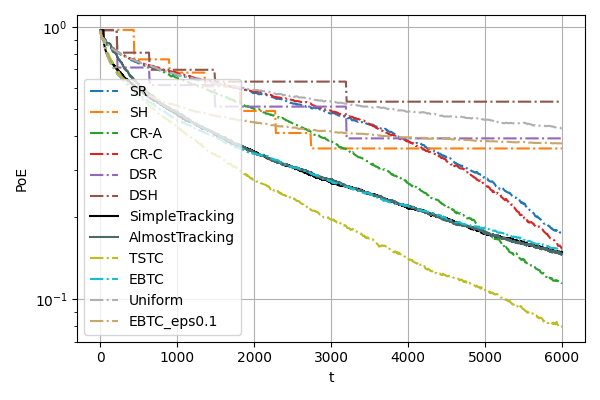}
\caption*{Instance 1: $\bP = \{1 - (i-1)\times0.05\}$} %
\includegraphics[width=0.8\textwidth]{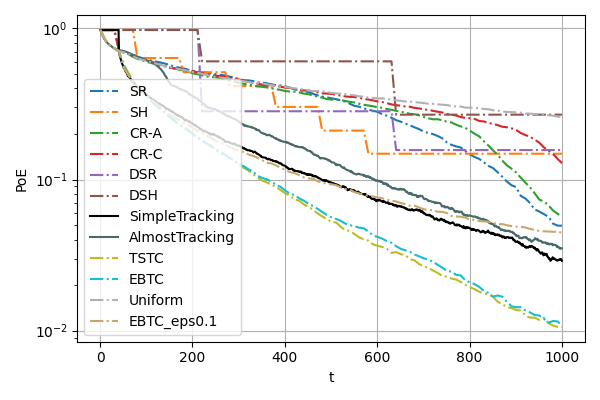}
\caption*{Instance 2: $\bP = \{ 10 \frac{(i-1)^{0.8}}{39^{0.8}} \}$} %
\end{figure}
\begin{figure}[htbp]
    \centering
\includegraphics[width=0.8\textwidth]{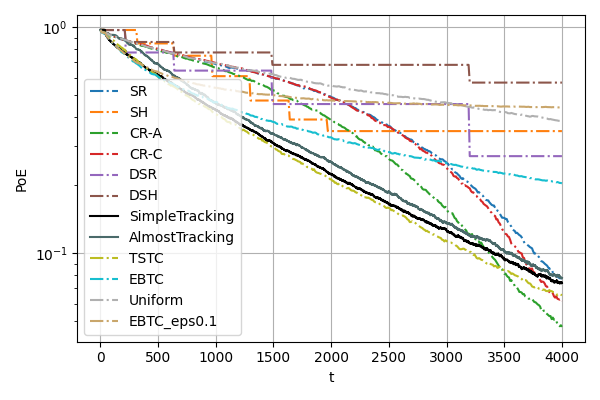}
\caption*{Instance 3: $\bP = \{1 - \sqrt{i-1}/10\}$. This instance is in favor of SR that we discussed in Section \ref{sec_srwin}.} %
\includegraphics[width=0.8\textwidth]{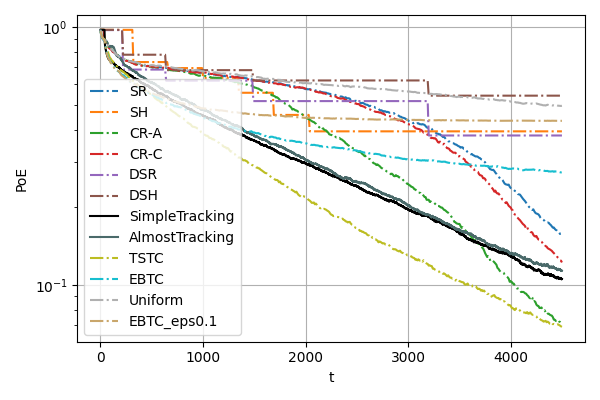}
\caption*{Instance 4: $\bP = \{1, 0.9, 0.9, 0.9, 0.9, \underbrace{0,0,\dots,0,0}_{\text{35 arms with zero mean}}\}$. This instance is in favor of Almost Tracking that we discussed in Section \ref{sec_ourwin}.} %
\end{figure}
\begin{figure}[htbp]
    \centering
\includegraphics[width=0.8\textwidth]{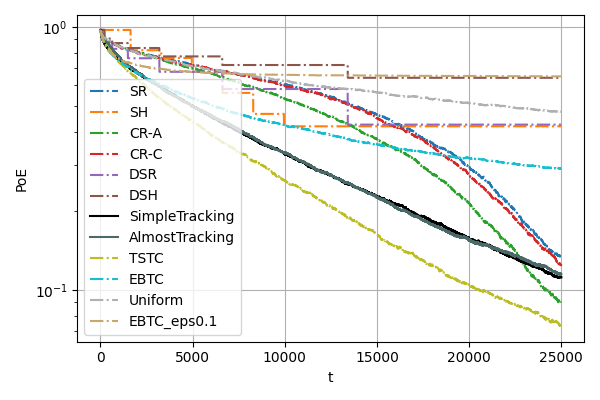}
\caption*{Instance 5: $\bP = \{ \sin((K-1)\pi/(2K))\} \cup \{\sin(9 \pi (K-i)/(20 K))\}_{i=2}^{40}$. This is \text{Concave} set of arms \citep{wang2023best}.} %
\includegraphics[width=0.8\textwidth]{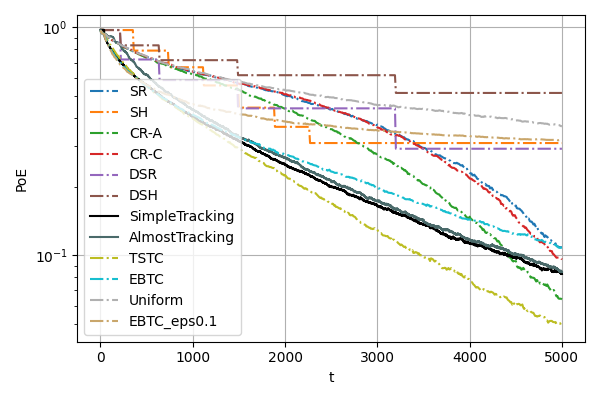}
\caption*{Instance 6: $\bP = \{0.75 \times 3^{-i/10}\}$. This is \text{Convex} set of arms \citep{wang2023best}.} %
\end{figure}
\begin{figure}[htbp]
    \centering
\includegraphics[width=0.8\textwidth]{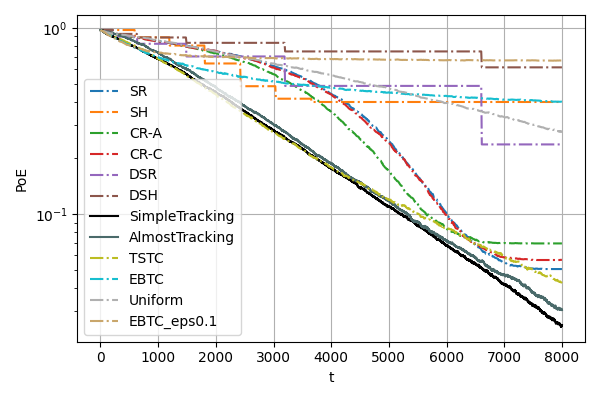}
\caption*{Instance 7: $\bP = \{1, \underbrace{0.8,0.8,\dots,0.8,0.8}_{\text{39 arms}} \}$.} %
\includegraphics[width=0.8\textwidth]{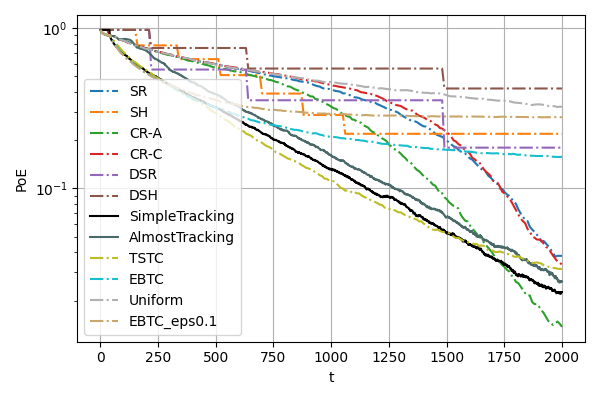}
\caption*{Instance 8: $\bP = \{1, \underbrace{0.8,0.8,0.8}_{\text{3 arms}}, \underbrace{0.8,0.8,\dots,0.8,0.8}_{\text{6 arms}}, \underbrace{0.2,0.2,\dots,0.2,0.2}_{\text{10 arms}}
, \underbrace{0,0,\dots,0,0}_{\text{20 arms}}\} $.} %
\end{figure}
\begin{figure}[htbp]
    \centering
\includegraphics[width=0.8\textwidth]{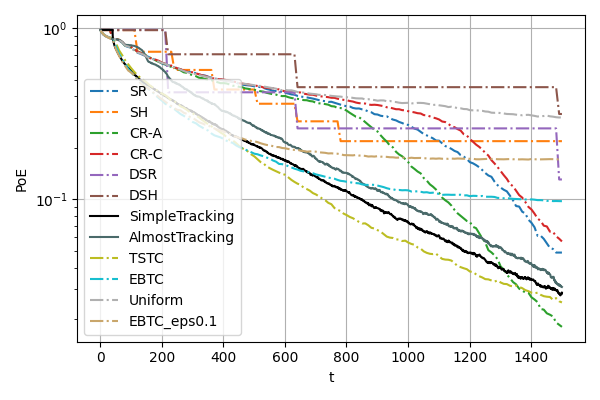}
\caption*{Instance 9: $\bP = \{1.0, 0.8, 0.8, \underbrace{0,0,\dots,0,0}_{\text{37 arms}}\}$.} %
\includegraphics[width=0.8\textwidth]{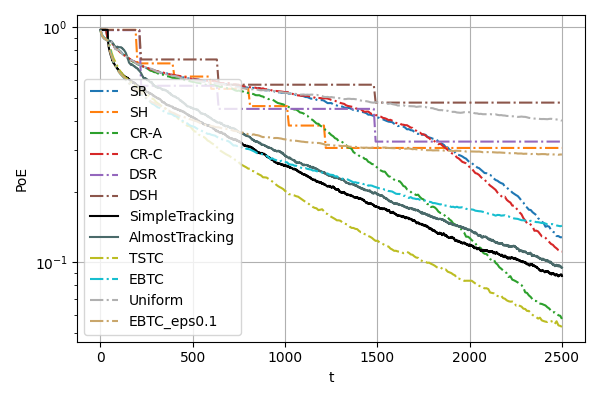}
\caption*{Instance 10: $\bP = \{1.0, 0.9, 0.85, 0.8, \underbrace{0,0,\dots,0,0}_{\text{36 arms}}\}$.} %
    
    \caption{Probability of Error (PoE) as a function of rounds $t$.}
    \label{fig: 10 model experiments}
\end{figure}
    
\clearpage

\subsection{Restatement of Table \ref{tab:minimax_rates} with confidence intervals}\label{subsec:confidence_intervals}

Performance of algorithms with 95\% confidence intervals are shown in Table \ref{tab:confidence_intervals} and \ref{tab:confidence_intervals_two}.

\begin{table}[ht]
    \centering
    \caption{Minimax rate for $H_1$ (lower, plugin, upper) for each algorithm. The column ``Plugin'' is based on the empirical average of PoE, which is the same as Table \ref{tab:minimax_rates} in the main paper. ``Lower Bound'' and ``Upper Bound'' are based on the lower and upper bounds of PoE with 95\% confidence intervals. One can see that Simple Tracking and Almost Tracking outperform the other algorithms with statistical significance.}
    \label{tab:confidence_intervals}
    \begin{tabular}{lccc}
    \hline
    Algorithm       & Lower Bound & Plugin  & Upper Bound \\
    \hline
SR & 0.443 & 0.456 &0.470 \\
SH & 0.216 & 0.222 &0.229 \\
CR-A & 0.423 & 0.435 &0.447 \\
CR-C & 0.304 & 0.312 &0.319 \\
DSR & 0.227 & 0.234 &0.242 \\
DSH & 0.111 & 0.114 &0.116 \\
Simple Tracking & 0.521 & 0.537 &0.555 \\
Almost Tracking & 0.494 & 0.509 &0.525 \\
TS-TC & 0.375 & 0.396 &0.424 \\
EB-TC & 0.177 & 0.181 &0.186 \\
Uniform & 0.154 & 0.158 &0.163 \\
EB-TC$_{\varepsilon_0}$ & 0.122 & 0.125 &0.128 \\
    \hline
    \end{tabular}
\end{table}

\begin{table}[ht]
    \centering
    \caption{Minimax rate for $H_2$ (lower, plugin, upper) for each algorithm. The column ``Plugin'' is based on the empirical average of PoE, which is the same as Table \ref{tab:minimax_rates} in the main paper. ``Lower Bound'' and ``Upper Bound'' are based on the lower and upper bounds of PoE with 95\% confidence intervals.}
    \label{tab:confidence_intervals_two}
    \begin{tabular}{lccc}
    \hline
    Algorithm       & Lower Bound & Plugin  & Upper Bound \\
    \hline
SR & 0.249 & 0.256 &0.263 \\
SH & 0.103 & 0.106 &0.108 \\
CR-A & 0.264 & 0.271 &0.279 \\
CR-C & 0.189 & 0.194 &0.199 \\
DSR & 0.128 & 0.131 &0.134 \\
DSH & 0.055 & 0.056 &0.058 \\
Simple Tracking & 0.253 & 0.260 &0.267 \\
Almost Tracking & 0.249 & 0.255 &0.262 \\
TS-TC & 0.216 & 0.227 &0.244 \\
EB-TC & 0.102 & 0.104 &0.107 \\
Uniform & 0.089 & 0.091 &0.093 \\
EB-TC$_{\varepsilon_0}$ & 0.070 & 0.072 &0.074 \\
\hline
    \end{tabular}
\end{table}

\subsection{Instances derived from real-world data}
\label{tab:realworld_instances}

In addition to the synthetic datasets, we adopt two real-world datasets to evaluate the performance of our algorithms. 
The Open Bandit Dataset \cite{SaitoOBP} is a publicly available, real-world logged bandit dataset collected a prominent Japanese fashion e-commerce platform (ZOZO, Inc.). It consists of 12 million impressions consists of 80 advertisements. 
We view each advertisements as a Gaussian arm and normalized the reward with average standard deviation so that the variance of each arm is 1.

The Movielens 1M dataset \cite{movielens1m} comprises over 1,000,000 anonymous ratings of nearly 3,900 movies by 6,040 MovieLens users who joined in 2000. 
Each datapoint includes user IDs, movie IDs, 5-star ratings, timestamps, and many other features.
In our experiments, we filter movies with 2,000 or less ratings to keep 31 popular movies.
We use each movie as an arm and its average rating as the mean reward, which is normalized by the average standard deviation of all movies so that the variance of each arm is 1.

\begin{table}[ht]
    \centering
    \caption{Open Bandit: CTR values extracted from the Open Bandit dataset.}
    \label{tab:openbandit_ctrs}
    \vspace{0.5em}
    \begin{tabular}{rrrrr}
    \hline
    0.0029265 & 0.0014464 & 0.0021134 & 0.0026464 & 0.0018947 \\
    0.0032350 & 0.0024874 & 0.0052780 & 0.0037272 & 0.0025919 \\
    0.0015018 & 0.0033327 & 0.0018368 & 0.0020283 & 0.0029336 \\
    0.0030222 & 0.0032011 & 0.0036364 & 0.0036137 & 0.0018426 \\
    0.0017718 & 0.0023036 & 0.0028038 & 0.0025506 & 0.0024710 \\
    0.0019308 & 0.0021782 & 0.0016784 & 0.0037885 & 0.0015287 \\
    0.0045120 & 0.0041963 & 0.0036784 & 0.0032292 & 0.0055569 \\
    0.0055678 & 0.0028800 & 0.0035584 & 0.0044478 & 0.0053337 \\
    0.0026211 & 0.0055760 & 0.0035852 & 0.0048702 & 0.0024826 \\
    0.0051337 & 0.0039318 & 0.0055106 & 0.0044275 & 0.0057023 \\
    0.0034024 & 0.0056714 & 0.0049135 & 0.0028941 & 0.0026866 \\
    0.0038009 & 0.0026913 & 0.0037623 & 0.0049876 & 0.0055036 \\
    0.0048012 & 0.0059725 & 0.0044809 & 0.0056396 & 0.0033993 \\
    0.0041044 & 0.0038471 & 0.0019121 & 0.0018957 & 0.0035998 \\
    0.0022913 & 0.0030215 & 0.0027332 & 0.0025879 & 0.0020447 \\
    0.0026221 & 0.0036932 & 0.0024460 & 0.0052332 & 0.0056697 \\
    \hline
    \end{tabular}
\end{table}

\textbf{Open Bandit Dataset:} Table \ref{tab:openbandit_ctrs} shows the CTR values extracted from Open Bandit dataset ($K=80$). Based on the CTR values, we derive the instance $P$ by normalizing the mean standard deviation ($= 0.057774753125$). We then modeled each draw as a result of Milli-impressions (1,000 impressions) and calculated the instance $\bP$ by multiplying the normalized CTRs by $\sqrt{10^3}$. We set $T=3000$ and run each simulation for \Runnum{} times.

\begin{table}[ht]
    \centering
    \caption{MovieLens 1M: Normalized ratings of movies with $>2000$ ratings.}
    \label{tab:movielens1m_ctrs}
    \vspace{0.5em}
    \begin{tabular}{rrrrr}
    \hline
    0.86074 & 0.79806 & 0.90208 & 0.79304 & 0.88125 \\
    0.82937 & 0.89074 & 0.86747 & 0.85094 & 0.68196 \\
    0.80458 & 0.84699 & 0.81170 & 0.86348 & 0.75277 \\
    0.79061 & 0.85860 & 0.89554 & 0.87036 & 0.86317 \\
    0.82550 & 0.91091 & 0.81759 & 0.82508 & 0.74799 \\
    0.83192 & 0.83041 & 0.85564 & 0.84388 & 0.78111 \\
    0.90499 &        &        &        &        \\
    \hline
    \end{tabular}
\end{table}
    
\textbf{Movielens 1M Dataset:} Table \ref{tab:movielens1m_ctrs} shows the normalized ratings of movies with $>2000$ ratings ($K = 31$). We normalize the ratings by the mean standard deviation of the ratings ($= 0.17820006619699696$) to obtain instance $\bP$. We set $T=10000$ and run each simulation for \Runnum{} times.

Note that the simulation results with these two instances derived from real-world data are shown in Figure \ref{fig:two_plots} in the main paper.

\subsection{Suboptimality of Two-two algorithm}\label{subsec:two_two_suboptimality}

In our simulation, we have tested two top-two algorithms, Thompson Sampling - Transportation
Costs (TS-TC) and Empirical Best - Transportation
Costs (EB-TC). These algorithms are designed for the fixed-confidence setting and have a zero rate in the fixed-budget setting.
To empirically demonstrate the suboptimality of the Two-two algorithms, we consider a version of Instance with different $T$. 
Table \ref{tab:toptworate} shows that the rate ($= (H(\bP)/T)\log(1/\PoE)$) for Top-Two keeps decreasing as $T$ increases, implying that its error probability only decays sub-exponentially.
In contrast, the PoE of other good algorithms (SR, CR, and Tracking) decays exponentially fast until it reaches below $1/\Runnum{}$.
\begin{table*}[h]
    \vspace{-0.2cm}
    \centering
    \caption{Estimated minimax rates of algorithms on instance 9 with different value of $T$ (larger better). Here, $\infty$ means that all \Runnum{} runs correctly identified the best arm.}
    \label{tab:toptworate}
    \begin{tabular}{c||cccc|cccc|cc}
    \toprule
    T & SR & SH & CR-A & CR-C & DSR & DSH & S.Track & A.Track & TS-TC & EB-TC \\
    \midrule
2000 & 0.637 & 0.321 & 0.915 & 0.671 & 0.366 & 0.200 & 0.785 & 0.773 & 0.721 & 0.415 \\
10000 & $\infty$ & 0.172 & $\infty$ & $\infty$ & 0.206 & 0.087 & $\infty$ & $\infty$ & 0.237 & 0.094 \\
20000 & $\infty$ & 0.132 & $\infty$ & $\infty$ & $\infty$ & 0.072 & $\infty$ & $\infty$ & 0.136 & 0.050 \\
    \bottomrule
\end{tabular}
\end{table*}

\clearpage

\subsection{An instance demonstrating the worst performance of continuous rejection sampling}\label{subsec:kyoungseok_cr}
Continuous Rejection (CR, \cite{wang2023best}) outperformed Successive Rejection (SR) in our main simulation (Table \ref{tab:minimax_rates}). Still, we found the performance of CR is more sensitive to the number of rounds $T$.
In particular, CR tends to have high PoE when \textbf{there are many similarly suboptimal arms} compared to the total number of samples. 
CR eliminates arms at each time step based on a significant gap between either the worst active arm and the second worst (in CR-C), or between the worst arm and the average of the remaining arms (in CR-A). It essentially performs uniform sampling across active arms. However, this strategy becomes inefficient when the second worst arm is similarly bad (CR-C), or when there are so many arms close in performance to the worst arm that the average does not differ significantly (CR-A). In such cases, CR struggles to remove suboptimal arms from the active set in a timely manner, leading to excessive and inefficient sampling.

The key bottleneck lies in the \textbf{timing of the first elimination }of a group of similarly suboptimal arms. If this elimination is delayed due to limited remaining time, CR ends up with uniform sampling, which leads to a high PoE.

The following experiments allow us to observe the limitations of CR more clearly. 
To highlight the limitation of CR, we set $\theta_0 = 1/\olog K$, which makes the timing of the first elimination at least the same or later than SR.
The results in Figure \ref{fig: modified 4 model experiments} show the outcome of best arm identification performed on the same bandit instance with a much shorter time horizon of $T = 1000$. One can check that CR-A and CR-C exhibit significantly higher error probabilities compared to the other algorithms, which can be largely attributed to the structural issue of CR mentioned earlier—namely, the delayed initiation of arm elimination.

In contrast, SR and Tracking are more robust against the choice of $T$. SR follows a predetermined schedule in which one arm is eliminated at each checkpoint based on its index $i$ and the total time $T$, ensuring that less promising arms are eliminated in a timely manner, maintaining reasonable sampling efficiency. Our Tracking algorithm, meanwhile, dynamically updates the sampling distribution based on the empirical means at each time step. Even when many poor arms are present, they are quickly assigned low distribution weights around the same time, allowing the algorithm to focus resources efficiently.

\begin{figure}[htbp]
\centering\includegraphics[width=0.8\textwidth]{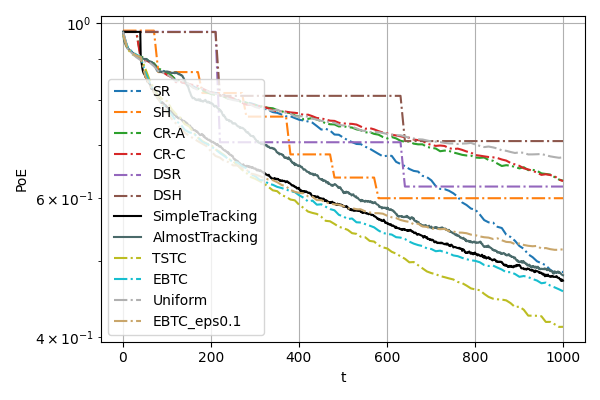}
\caption*{Modified Instance 1: $\bP = \{1 - (i-1)\times0.05\}$ with $T=1000$} %
\includegraphics[width=0.8\textwidth]{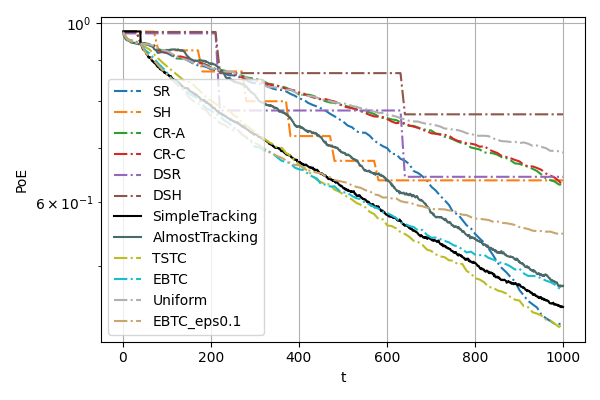}
\caption*{Modified Instance 3: $\bP = \{1 - \sqrt{i-1}/10\}$ with $T=1000$.} %
\end{figure}
\begin{figure}[htbp]
    \centering
\includegraphics[width=0.8\textwidth]{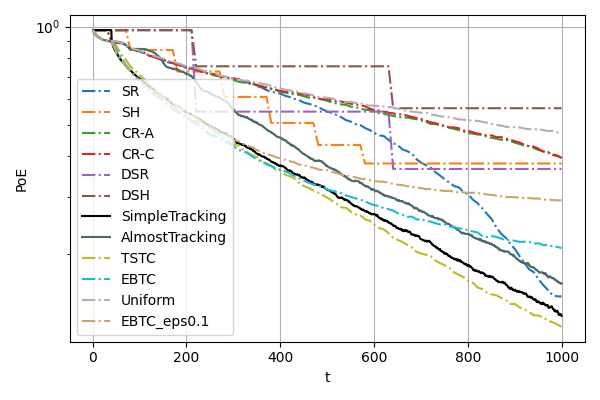}
\caption*{Modified Instance 8: $\bP = \{1, \underbrace{0.8,0.8,0.8}_{\text{3 arms}}, \underbrace{0.8,0.8,\dots,0.8,0.8}_{\text{6 arms}}, \underbrace{0.2,0.2,\dots,0.2,0.2}_{\text{10 arms}}
, \underbrace{0,0,\dots,0,0}_{\text{20 arms}}\} $ with $T=1000$.} %
\includegraphics[width=0.8\textwidth]{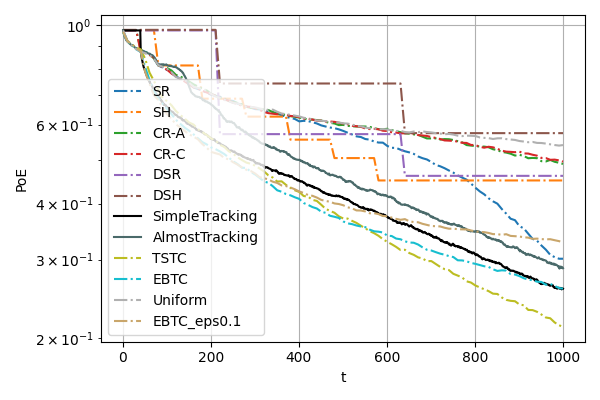}
\caption*{Modified Instance 10: $\bP = \{1.0, 0.9, 0.85, 0.8, \underbrace{0,0,\dots,0,0}_{\text{36 arms}}\}$ with $T=1000$.} %
    
    \caption{Probability of Error (PoE) as a function of rounds $t$, under the same setting as Figure \ref{fig: 10 model experiments} but with a shorter time horizon ($T=1000$).} \label{fig: modified 4 model experiments}
\end{figure}
\clearpage

\end{document}